\documentclass[oneside,11pt]{article} 

\usepackage[utf8]{inputenc}
\usepackage[T1]{fontenc}
\usepackage{microtype}

\usepackage{amsmath,amsfonts,amssymb,amsthm,commath}
\usepackage{algorithm,algorithmic}
\usepackage{booktabs,enumitem}
\usepackage{graphicx}
\usepackage{url}
\usepackage{xcolor}
\usepackage[backref=page]{hyperref}
\usepackage{balance}
\usepackage{pifont}
\usepackage{xspace}

\usepackage{comment}   
\usepackage{dsfont}    
\usepackage{float}

\newcommand{\R}{\mathbb{R}}

\newcommand{\twonorm}[1]{\left\lVert #1 \right\rVert}

\renewcommand{\abs}[1]{\left\lvert #1 \right\rvert}

\DeclareMathOperator*{\argmin}{arg\,min}
\DeclareMathOperator*{\argmax}{arg\,max}

\renewcommand{\Pr}{\operatorname{Pr}}

\newcommand{\poly}{\operatorname{poly}}

\newcommand{\err}{\mathrm{err}}

\newcommand{\E}{\mathbb{E}}          
\newcommand{\I}{\mathbb{I}}          

\newcommand{\B}{\mathcal{B}}
\newcommand{\D}{\mathcal{D}}

\renewcommand{\H}{\mathcal{H}}

\renewcommand{\P}{\mathcal{P}}

\newcommand{\X}{\mathcal{X}}
\newcommand{\Y}{\mathcal{Y}}

\newcommand{\one}{\mathds{1}}

\newcommand{\inner}[2]{\left\langle #1,\, #2 \right\rangle}

\newcommand{\tnorm}[1]{\left\|#1\right\|_2}

\renewcommand{\(}{\left(}
\renewcommand{\)}{\right)}
\renewcommand{\[}{\left[}
\renewcommand{\]}{\right]}

\newcommand{\mcPsi}{\Psi}
\newcommand{\mcY}{[k]}
\newcommand{\mcX}{\R^d}
\newcommand{\mcFeat}{\R^{kd}}
\newcommand{\Phat}{{P_{\hat{w}}^\tau}}

\newcommand{\EXnas}{\text{EX}_\text{nas}}

\DeclareMathOperator{\maj}{maj}

\theoremstyle{plain}
\ifx\theorem\undefined
\newtheorem{theorem}{Theorem}
\fi
\ifx\lemma\undefined
\newtheorem{lemma}[theorem]{Lemma}
\fi
\ifx\proposition\undefined
\newtheorem{proposition}[theorem]{Proposition}
\fi
\ifx\corollary\undefined
\newtheorem{corollary}[theorem]{Corollary}
\fi

\theoremstyle{definition}
\ifx\remark\undefined
\newtheorem{remark}[theorem]{Remark}
\fi
\ifx\definition\undefined
\newtheorem{definition}[theorem]{Definition}
\fi
\ifx\conjecture\undefined
\newtheorem{conjecture}[theorem]{Conjecture}
\fi
\ifx\fact\undefined
\newtheorem{fact}[theorem]{Fact}
\fi
\ifx\claim\undefined
\newtheorem{claim}[theorem]{Claim}
\fi
\ifx\assumption\undefined
\newtheorem{assumption}{Assumption}
\fi

\numberwithin{equation}{section}

\author{
	{Rita Adhikari}\\
	Augusta University\\
	\texttt{riadhikari@augusta.edu}
	\and
	Shiwei Zeng\\
	Augusta University\\
	\texttt{szeng@augusta.edu}
}

\title{Efficient and Noise-Tolerant PAC Learning of Multiclass Linear Classifiers}

\begin{document}
	\maketitle

\begin{abstract}
Noise-tolerant PAC learning of linear models has been of central interests in machine learning community since the last century. In recent years, many computationally-efficient algorithms have been proposed for the problem of learning linear threshold functions under multiple noise models. Yet, when the problem is considered under multiclass learning settings, i.e. when the number of classes $k$ is at least $3$, it is unknown whether there exist computationally-efficient PAC learning algorithms when the data sets are maliciously corrupted. In this paper, we consider that the marginal distribution is a mixture of bounded variance distributions and the data sets satisfy a margin condition at the same time. We show that there exists a computationally-efficient algorithm that PAC learns multiclass linear classifiers $\{h_w:x\mapsto \arg\max_{y\in[k]}w_y\cdot x, x\in \mathbb{R}^d, w\in\mathbb{R}^{kd}\}$ using at most $O(k^2\cdot (d\log d+\log k))$ samples even under a constant rate of nasty noise. Our algorithm consists of two main ingredients: a cluster-based pruning scheme and a standard multiclass hinge loss minimization program. Even in the special case of binary setting, i.e. $k=2$, our result is strictly stronger than all prior works. 

\end{abstract}

\section{Introduction}


This paper revisits the problem of learning the multiclass linear classifiers (MLCs) under noise corruptions. Robust learning problems have been extensively studied under the binary learning scenarios. However, our understanding of when and how we can achieve algorithmic robustness under the multiclass learning settings, i.e. $k>2$, is falling short. Notably, it was shown recently in~\cite{diakonikolas2025statistical} that distribution-free learning of multiclass linear classifiers under the random classification noise suffers from a super-polynomial statistical query lower bound, while the problem is polynomial-time learnable when $k=2$~\cite{blum1998polynomial}. This gap shows that efficient robust learning in the multiclass settings differs drastically from that in the binary settings. Hence, it is natural to ask: 
\begin{quote}
Does there exist natural condition under which the algorithmic robustness established in binary settings carry over to the multiclass settings?
\end{quote}

We answer the above question in the affirmative by considering the conditions being investigated in recent work of~\cite{talwar2020error,shen2025efficient}. That is, they have shown that if we incorporate two widely used distribution conditions, i.e. concentration and margin conditions, there exists computationally efficient algorithms that probably approximately correctly (PAC) learn the target halfspace even with a constant fraction of corruptions under the challenging malicious noise model~\cite{kearns1988malicious}. This result is surprising as there exists known information-theoretical lower bound of $\frac{\epsilon}{1+\epsilon}$ for distribution-free learning of halfspaces under malicious noise corruptions. Even with concentration distributional conditions alone, the best known malicious noise tolerance was $\Theta(\epsilon)$ for halfspaces~\cite{awasthi2017power}. For learning of MLCs, a recent work of~\cite{diakonikolas2025robust} proposes an algorithm with noise tolerance of $\Theta(\epsilon)$ under the Gaussian marginals in the presence of label noise. It is not clear whether significant improvements in noise tolerance are achievable without introducing new conditions or techniques. However, incorporating the concentration and the margin conditions can immediately improve the noise tolerance drastically to a constant rate for challenging noise models such as adversarial label noise~\cite{talwar2020error} and malicious noise~\cite{shen2025efficient}. Moreover, the algorithmic design is simple, i.e. surrogate loss minimizaton with spectral filtering schemes, and the theoretical framework is elegant.

In this work, we extend this framework to multiclass learning settings. Specifically, let $\X\subseteq\R^d$ be the instance space, and $\Y=\{1,2,\dots,k\}$ be the label space. We consider the class of MLCs $\H:=\{h_{W}: x\mapsto \arg\max_{y\in \Y}W_y x, W\in\R^{k\times d},y\in[k]\}$
and assume that the marginal distribution $\D_\X$ is a mixture of bounded variance distributions or logconcaves. Moreover, we assume that any empirical sample set from the underlying distribution $\D$ over $\X\times\Y$ satisfies a margin condition. 
In our analysis, we focus on the most challenging nasty noise model.

\begin{definition}[PAC learning of MLCs with nasty noise]\label{def:nasty}
Let $\epsilon,\delta\in(0,1)$ be the target error rate and confidence parameter, respectively.
A nasty adversary $\EXnas(\D_\X,W^*,\eta)$ fixes $\D_\X,W^*$ and $\eta\in(0,\frac{1}{k})$ throughout the learning process, takes a request from  the learner, and first draws $n$ instances independently from $\D_\X$ and labels them by $W^*$, forming a clean sample set $S_C=\{(x_i,\argmax_{y\in\Y}W_y^* x_i)\}_{i=1}^{n}$. The adversary, fully aware of the learning algorithm, may then inspect the set $S_C$ and replace at most an $\eta$ fraction of it with any instance-label pairs from $\X\times\Y$. The resulting set $S$ is returned to the learner. The goal of the learner is to output a multiclass linear classifier $\widehat{W}\in\R^{k\times d}$ such that with probability $1-\delta$ (over the randomness of the samples and all internal random bits of the learning algorithm), $\Pr_{x\sim\D_\X}(\argmax_{y\in\Y} {\widehat{W}_y x} \neq \argmax_{y\in\Y}{W_y^* x}) \leq \epsilon$.
\end{definition}

\subsection{Main results}\label{sec:result}

To ensure efficient learnability, we rely on the following structural assumptions regarding the margin and the marginal distribution. Note that the parameter $K$ we used in Assumption~\ref{ass:mixture} and \ref{ass:log_concave} does not necessarily equal to the number of classes $k$.

\begin{assumption}[Multiclass $\gamma$-margin condition]
	\label{ass:large_margin}
	There exists a vector $w^* \in \mcFeat$ with $\tnorm{w^*}\leq1$ and a margin $\gamma > 0$ such that for any clean sample $(x, y)$ in the support of $\D$:
	\begin{equation}
		\min_{y' \in \Y \setminus \{y\}} \inner{w^*}{\mcPsi(x, y) - \mcPsi(x, y')} \geq \gamma .
	\end{equation}
\end{assumption}

\begin{assumption}[Mixtures of distributions]
	\label{ass:mixture}
	The marginal distribution $\D_X$ is a mixture of $K$ distributions $\D_1, \dots, \D_K$, i.e., $\D_X = \frac{1}{K} \sum_{j=1}^K q_j\D_j$. Furthermore, for each component $j \in [K]$, $\D_j$ is a distribution over $\X$ with mean $\mu_j = \E_{x \sim \D_j}[x]$ and covariance matrix $\Sigma_j = \E_{x \sim \D_j}[(x-\mu_j)(x-\mu_j)^\top]$ satisfying $\Sigma_j \preceq \sigma^2\I_d$.
\end{assumption}
\begin{remark}[Generalization of assumptions]
The prior works of~\cite{talwar2020error,shen2025efficient} also assumed both the margin and mixture of logconcaves conditions. However, our setting is much more general with any parameter $\sigma$, as long as the margin parameter $\gamma$ is made compatible to it. In contrast, their algorithm works only for $\sigma=\frac{1}{\sqrt{d}}$.
\end{remark}
Though our algorithm and analysis work for any mixture of bounded covariance distributions, the assumption of logconcaves still contribute to a more favorable margin condition.
\begin{assumption}[Mixture of logconcaves]
	\label{ass:log_concave}
	Consider Assumption~\ref{ass:mixture}, we say that $\D_X$ is a mixture of logconcave distributions if $\forall j\in[K]$, $D_j$ is a logconcave distribution with mean $\mu_j$ and $\Sigma_j \preceq \sigma^2\I_d$.
\end{assumption}

We state our main results below. We show that obtaining PAC gaurentee for learning of MLCs while tolerant to a constant fraction of nasty noise is possible under the margin condition and the mixture of bounded covariance distributions. More importantly, we show that such guarantees can be achieved with computational efficiency and a sample complexity of at most $\tilde{O}(k^2d)$. Note that under the conventional assumption in this line of research, the bound  $1/k^2$ in noise rate is considered as $ \Theta(1)$.

\begin{theorem}[Main theorem, mixture of bounded covariance]\label{thm:main-boundedcov}
Let $C>0$ be some universal constant. Suppose Assumption~\ref{ass:mixture} holds with $\alpha\in[\frac{0.6}{k},\frac{1}{k}]$, and  Assumption~\ref{ass:large_margin} hold with some $\gamma > \max(4\tau,4C\sigma)$, where $\tau=2\sigma\cdot\sqrt{\frac{k}{\epsilon}}$. 
By taking at least $\Omega(k^2\cdot (d\log d+\log\frac{k}{\epsilon\delta}))$ samples from $\EXnas$ with $\eta\leq\frac{1}{2^{12}k^2}$, Algorithm~\ref{alg:svm} returns a $\hat{w}$ in polynomial time such that $\err_{\D_X}(\hat{w})\leq\epsilon$ with probability at least $1-\delta$.
\end{theorem}
The above theorem shows that as long as the clean samples are separable with a $\Omega\big(\sigma\cdot\sqrt{{k}/{\epsilon}}\big)$-margin, it is robustly learnable even under the nasty noise condition in polynomial time.
We show that if additionally, the marginal distribution satisfies a mixture of logconcave condition, similar guarantees can be achieved for a smaller $\gamma=\Omega(\sigma\cdot(\log k/\epsilon +1))$.
\begin{theorem}[Main theorem, mixture of logconcaves]\label{thm:main-logconcave}
Let $C>0$ be some universal constant. Suppose Assumption~\ref{ass:log_concave} holds with $\alpha\in[\frac{0.6}{k},\frac{1}{k}]$ and  Assumption~\ref{ass:large_margin} hold with some $\gamma > \max(4\tau,4C\sigma)$, where $\tau=2\sigma(\log\frac{k}{\epsilon}+1)$. 
By taking at least $\Omega(k^2\cdot (d\log d+\log\frac{k}{\epsilon\delta}))$ samples from $\EXnas$ with $\eta\leq\frac{1}{2^{12}k^2}$, Algorithm~\ref{alg:svm} returns a $\hat{w}$ in polynomial time such that $\err_{\D_X}(\hat{w})\leq\epsilon$ with probability at least $1-\delta$.
\end{theorem}
\begin{remark}[Tail bounds]
It is well known that bounded covariance distributions have much heavier probability tail bounds comparing to logconcaves. It is surprising that our main theorems achieve exactly the same sample complexity under both conditions. That means the heavy tails do not make learning of MLCs harder under only the mixture of bounded covariances and margin conditions in terms of statistical and computational results. The major difference is on the margin parameter $\gamma$, which is assumed to be larger than the empirical cluster radius ($\tau$). This is intuitive, as any empirical set from a bounded covariance distribution naturally has a larger radius than that from a logconcave.
\end{remark}
\begin{remark}[Non-uniform weight mixtures]
Prior works studying similar settings assume the marginal distribution to be a uniformly weighted mixture of distributions. In our work, due to the implementation of a strong noise-tolerant clustering algorithm of~\cite{diakonikolas2025clustering}, our algorithm is able to learn under non-uniformly weighted mixtures. However, since the weights also perform as an active upper bound on the noise rate, in our setting, the interesting regime is still constant weights.
\end{remark}
\begin{remark}[Implication for $k=2$]
Even for the binary case, our algorithm and analysis are still stronger than all existing results. Until this work, it was unknown whether the conditions can be utilized to learn with a constant rate of nasty noise. This is because previous analysis essentially relies on the existence of a set of uncorrupted i.i.d. samples, which is possible under the malicious noise model (\cite{shen2025efficient}). In this work, we incorporate a much more powerful outlier-removal scheme such that even the clean samples are no longer i.i.d., we are still able to recover essential properties for robust learning.
\end{remark}



\subsection{Overview of techniques}\label{sec:techniques}

{\bfseries Algorithmic design.}\ 
Our algorithm builds on two main ideas for robustly learning multiclass linear classifiers: pruning the mixture distribution using clustering techniques, and applying hinge loss minimization on the dataset obtained after pruning. 
Algorithm~\ref{alg:svm} begins with a hard pruning subroutine, i.e. Algorithm~\ref{alg:hard_purging}, that returns clusters $B_j$ such that each cluster has bounded empirical covariance with its mean close to the true cluster means, and at most a small fraction of points are misassigned. It is worth noting that the algorithm finishes with a filtered Voronoi partition, such that each $B_j$ returned by it would satisfy a bounded empirical covariance condition.
Next, we perform a cluster-wise label based pruning scheme, detecting the majority label inside each cluster and removing any points with labels that differ from it. As each cluster has the same label for each data point inside it, this property is later exploited for controlling the dirty samples. We then apply hinge loss minimization and analyze the gradient contributions from clean and corrupted samples relative to the program optimums. 
We bound the contribution coming from corrupted samples by using the bounded empirical covariance condition guaranteed by the filtered Voronoi partition of Algorithm~\ref{alg:hard_purging}.

{\bfseries Gradient and subgradient analysis for multiclass.}\ 
The multiclass hinge loss function is a strict generalization of the binary hinge loss. However, when handling $k>2$ classes, the case is much more complicated than its binary counterpart (see Eq.~\eqref{eq:hinge-loss}). Intuitively, the dirty samples can not only drag the optimizer towards an opposite class, but any of the remaining $k-1$ classes. As the hinge loss carefully selects a label for which the loss function is maximized, our analysis needs also to be adapted based on that and make sure the algorithm is robust on all directions.

{\bfseries Sample complexity for generalized conditions.}\ 
In this paper, we have adapted the previous analysis on mixture of logconcaves with covariance parameter $\sigma=\frac{1}{\sqrt{d}}$ to any $\sigma$, as long as the margin parameter is made compatible with it. Moreover, it was unknown whether the algorithmic robustness still holds under only the bounded second-order moment conditions instead of logconcave. We show that comparable results can be established for bounded covariance with slightly worse margins.

\section{Related Work}


Multiclass linear classification (MLC) is learning problem in machine learning ~\cite{shwartz2014understanding} and has been studied extensively in the batch, online and bandit settings, both theoretically and empirically ~\cite{platt1999large, tewari2007consistency, kakade2008efficient,	beygelzimer2019bandit}. MLC can be solved efficiently by linear programming or the multiclass Perceptron with sample complexity $O(dk/\epsilon)$~\cite{shwartz2014understanding} in the noiseless realizable setting. The corresponding linear program can be solved efficiently in the Statistical Query (SQ) model~\cite{kearns1998efficient}. The sample complexity of MLC under RCN is bounded by $\min\{\tilde{O}(dk/(\sigma\epsilon)), \tilde{O}(dk/\epsilon^2)\}$ samples via empirical risk minimization, as derived from~\cite{massart2006risk}, where $H$ is noise matrix and $\zeta := \min_{i \neq j} H_{ii} - H_{ij}$ indictates the identifiability of the noise model~\cite{rooyen2018theory}. This setting is related to ours by viewing the inner product between the weight vector and instance being $H_{ii}$. ~\cite{wang2024unified} studied extensions of margin-based surrogate loss functions from binary to multiclass learning, and proposed multiclass SVM with maximizing minimum margin in~\cite{nie2024multiclass}.

 Prior research studies  multiclass  linear classifier under Random Classification Noise (RCN) from a number of criteria~\cite{kakade2008efficient, long2010random, patrini2017making, lipton2018detecting, rooyen2018theory, beygelzimer2019bandit, fotakis2021efficient, zhang2021learning}. Even for $k=2$~\cite{diakonikolas2021optimality,diakonikolas2023near}, recent hardness results indicate that obtaining $\mathrm{OPT}+\epsilon$ may require upto $d^{\mathrm{poly}(1/\epsilon)}$ time, while our distributional setting yields an error at most $\epsilon$ under a constant rate nasty noise tolerance.
 While SQ lower bounds demonstrates that optimal error learning under RCN require $d^{\Omega(k)}$ time.~\cite{diakonikolas2025statistical} proved super polynomial SQ hardness for distribution free MLC under RCN even when $k=3$, ~\cite{diakonikolas2025robust} provided a fixed degree polynomial time learner under Gaussian marginals gaining $O(\mathrm{OPT})+\epsilon$ error for $k\ge 3$. When taken as a whole, these findings support  the assumption that are robust enough to enable for effective learning but weaker than Gaussian structure.
 
%

Since the seminal works of~\cite{weston1999support,crammer2001multiclass} multiclass support vector machines and hinge loss minimization have been researched; expansions include GenSVM~\cite{burg2016gensvm} and convex outlier ablation integrated into large margin training~\cite{xu2006robust}. In binary case closest prior work is~\cite{talwar2020error} that first proved vanilla hinge loss minimization achieves $\Omega(\gamma)$ malicious noise tolerance and $\Omega(1)$ adversarial label noise tolerance, under large-margin and distributional assumptions on mixtures of log-concave distributions, with the key idea  of  controlling the gradient norm of the hinge loss results robustness. ~\cite{zeng2023attribute} focuses on attribute-efficiency for polynomial threshold functions under nasty noise. ~\cite{frei2021agnostic} extended to the agnostic setting and showed the multiplicative factors on the best achievable error rate depend on a soft-margin condition. ~\cite{shen2025efficient} recently achieved constant malicious noise tolerance via reweighted hinge loss building on~\cite{talwar2020error}.\cite{zeng2025attribute} shows that attribute-efficient learning of sparse halfspaces is possible at a constant rate of malicious noise.

Among most of existing clustering method needs stronger structure or do not offer the necessary robustness. For example,~\cite{chang2007clustering} considers log-concave mixtures with a nonparametric EM but lacks finite sample guarantees under separation or adversarial noise contamination. Similarly, other  works obtain polynomial-time clustering guarantees using higher moments, Sum-of-Squares methods, Poincar\'{e}-type assumptions, Gaussian structure or spectral decompositions~\cite{hopkins2018mixture,kothari2018robust,liu2022clustering,vempala2004spectral}. These  approaches are powerful, but they make stronger assumptions than ours, or are not designed  robustly for hinge loss learning under  noisy model. The closest to our setting is~\cite{diakonikolas2025clustering} which gives robust guarantees for clustering mixtures of distributions with bounded-covariance  under the separation condition $\|\mu_i-\mu_j\|_2 \gtrsim (\sigma_i+\sigma_j)/\sqrt{\alpha}$.


To our study, no previous work provides provable constant noise tolerance for multiclass hinge loss minimization under  nasty noise for $k \geq 3$. Under the binary case, recently ~\cite{shen2025efficient} has obtained constant noise tolerance under malicious noise using reweighted hinge loss minimization using soft outlier removal technique  building on ~\cite{talwar2020error}. Our work is the first to extend hinge loss minimization with constant noise tolerance to the multiclass setting under nasty noise, combined with clustering of mixture models under bounded variance, large margin assumptions and log-concave distributions. In addition to that our work considers for both uniform and arbitrary weight mixtures whereas previous work for binary case provide gaurentees only for uniform weight mixtures~\cite{talwar2020error,shen2025efficient}.

	\section{Preliminaries}

Hinge loss minimization for multiclass linear classifiers is a textbook algorithm. In the following, we revisit \cite{shwartz2014understanding} and formally define our notations.
We use $[k]$ to denote the indexes $\{1,\dots,k\}$. Denote by $\inner{u}{v}$ the inner produce between two vectors $u,v$.
Recall that $W$ is a $k\times d$ matrix and we use $W_y$ to denote the $y$-th row of it. For the ease of presentation, we vectorize $W$ to a vector $w\in\R^{kd}$ and perform a feature mapping $\Psi:\X\times\Y\rightarrow \R^{kd}$ for any instance-label pair $(x,y)$, such that $\inner{W_y}{x}$ can be written as $\inner{w}{\Psi(x,y)}$. We use $W_y$ and $w_y$ interchangeably.


\textbf{Feature mapping and predictors.}
We employ a class-sensitive feature mapping $\mcPsi: \X \times \Y \to \mcFeat$ that projects an instance-label pair $(x, y)$ into a $k\times d$-dimensional feature vector in the following manner. The $((y-1)d+1)$-th element to $(yd)$-th element of $\mcPsi(x,y)$ is a copy of $x$; while all remaining elements are set to zeros. In other words,
\begin{equation}
h_w(x) = \argmax_{y \in \Y} \inner{w}{\mcPsi(x, y)} = \argmax_{y\in\Y} \inner{W_y}{x} = h_W(x) .
\end{equation}
In other words, the two presentations are equivalent.


\textbf{Generalized hinge loss.}
We employ the hinge loss function with an extension to the multiclass setting. That is, denote by $\one(\cdot)$ the indication function. The generalized hinge loss is $\ell(w; (x, y)) := \max_{y' \in \Y} \( \one(y\neq y') - \inner{w}{\mcPsi(x, y) - \mcPsi(x, y')} \)$. Note that if $k=2$ and let $\mcPsi(x, y)=\frac{yx}{2}$, this gives a standard hinge loss function for the binary case.
To fit the program the setting of the margin condition, we will use a scaled version of the hinge loss as follows
\begin{equation}\label{eq:hinge-loss}
	\ell_\gamma(w; (x, y)) := \max_{y' \in \Y} \( \one(y\neq y') - \inner{\frac{w}{\gamma}}{\mcPsi(x, y) - \mcPsi(x, y')} \).
\end{equation}
We further define the loss over a sample set $S$ as 
\begin{equation}
	\ell_\gamma(w; S) := \sum_{(x,y)\in S} \max_{y' \in \Y} \( \one(y\neq y') - \inner{\frac{w}{\gamma}}{\mcPsi(x, y) - \mcPsi(x, y')} \).
\end{equation}
We omit the subscript $\gamma$ when it is clear in the context.

The robustness of our algorithm essentially comes that the instances from one class lie close to each other, forming a region termed the dense pancake, which is defined as follows. Note that this definition is significantly changed from its prior works.

\begin{definition}[Multiclass pancake]
	\label{def:multiclass_pancake}
	Let $(x, y) \in \X \times \Y$. Given a unit vector $w \in \mcFeat$ and a parameter $\tau > 0$, the multiclass pancake $P_w^\tau(x, y)$ is defined as a set
\begin{equation}
P_w^\tau(x, y) := \left\{ (x', y') \in \X \times \Y : y'=y \wedge \forall \bar y\in\Y, \left| \inner{w_{\bar y}}{x' - x} \right| \leq \tau \right\}.
\end{equation}
	We say that the pancake $P_w^\tau(x, y)$ is $\rho$-dense with respect to a distribution $\D$ on $\X \times \Y$ if
	\begin{equation}
		\Pr_{(x', y') \sim \D} \left( (x', y') \in P_w^\tau(x, y) \right) \geq \rho.
	\end{equation}
	Let $\D_1, \D_2$ be two distributions on $\X \times \Y$. We say that the pair $(\D_1, \D_2)$ satisfies the $(\tau, \rho, \beta)$-dense pancake condition if for any unit vector $w \in \mcFeat$,
	\begin{equation}\label{eq:pancake-all-w}
		\Pr_{(x, y) \sim \D_2} \left( P_w^\tau(x, y) \text{ is } \rho\text{-dense w.r.t. } \D_1 \right) \geq 1 - \beta.
	\end{equation}
\end{definition}
For the ease of presentation, we alternatively say that $\D_1$ satisfies $(\tau,\rho,\beta)$-dense pancake condition with respect to $\D_2$.

Instead of measuring the density w.r.t. one direction, we require the pancake to be dense w.r.t. $k$ different directions. Each direction $w_{\bar y}$ is the weight vector for one class. In principle, all weight vectors in $W^*$ should be quite different from each other, as otherwise the two classes can be seemed as one. Note that this is different from Eq.~\eqref{eq:pancake-all-w} that requires the density condition to hold for all $w\in\mathbb{B}^{kd}$. This is because Eq.~\eqref{eq:pancake-all-w} measures not only one pancake, but any random pancake with respect to $(x,y)$ drawn from $\D_2$. Also note that, the sample complexity for achieve this multiclass pancake condition under the margin and mixture of distributions assumptions has changed. As a result, we have Theorem~\ref{thm:sample-comp-mixture}.

{\bfseries Additional notations.}\ We use $\mu_{S}$ and $\sigma_{S}$ to denote the empirical mean and the empirical standard deviation of a set $S$. We slightly abuse the notation and use $(x,y)\sim S$ to denote drawing a sample uniformly at random from a set $S$. In this sense, $S$ may also denote the uniform distribution on its elements.
We denote by $(\D_X,w^*)$ a distribution over $\X\times\Y$ where the instances are drawn from $\D_X$ and then labeled by $w^*$. 

	\section{Algorithm}
Our main algorithm consists of two subroutines. The first one is a cluster-based pruning algorithm that utilizes the mixture clustering techniques based on bounded covariance conditions with optimal separation~\cite{diakonikolas2025clustering}. The second ingredient is a standard multiclass hinge loss minimization over the pruned sample set. Due to the introduction of feature mapping $\Psi$, the program is a strict generalization of its binary counterpart, i.e. when $\Psi(y,x)=\frac{yx}{2}$, $\ell_\gamma(w;(x,y)) = \max \(0,1-\frac{w}{\gamma}\cdot yx\)$.

%

\begin{algorithm}[h]
	\caption{Multiclass SVM}
	\label{alg:svm}
	\begin{algorithmic}[1]
		\REQUIRE Training set $S = \{(x_1, y_1), \dots, (x_n, y_n)\} \subseteq\X\times\Y$, 
		parameter $\alpha \in (0,1)$, margin parameter $\gamma > 0$,
		{feature mapping} $\Psi: \mathcal{X} \times \mathcal{Y} \to \mathbb{R}^{dk}$.
		\ENSURE A multiclass linear classifier $h_w$.
		
		\STATE $\{B_{j'}\}_{j'\in[m]} \leftarrow$ Algorithm~\ref{alg:hard_purging}$(S,\alpha)$. \label{step:alg-pruning}
		
		\STATE $\forall j'\in[m]$, $B'_{j'} \leftarrow \{(x,y) \in B_{j'}: y=\maj(B_{j'})\}$.\label{step:label-pruning}
		
		\STATE For $j\in[k]$, let $\mathcal{B}_j = \{(x,y)\in \cup_{j'\in[m]}B_{j'}: y=j\}$. \label{step:cluster-merge}
%
		
		\STATE Let $\widehat{S} = \cup_{j\in[k]}B_{j}$. \label{step:combine}

		\STATE {Solve:}
		\begin{equation}\label{eq:svm}
			\min_{w\in\R^{dk}:\|w\|_2\le 1}\ 
			\sum_{(x_i,y_i)\in\widehat S}
			\max_{y'\in\Y}\Bigl(\one(y'\neq y_i) - \inner{\tfrac{w}{\gamma}}{\Psi(x_i,y_i)-\Psi(x_i,y')} \Bigr).
		\end{equation}
		
		\STATE {Return} $h_w(x) = \text{argmax}_{y \in \mathcal{Y}} \langle w, \Psi(x, y) \rangle$.
	\end{algorithmic}
\end{algorithm}

We first introduce the cluster-based pruning subroutine. We include it in Algorithm~\ref{alg:hard_purging} for easy references. It  performs robust clustering under the strong contamination model~\cite{diakonikolas2025clustering}, which  can be thought of as equivalent to the nasty noise model from the feature corruption point of view.

\begin{algorithm}[h]
\caption{Clustering (Algorithm~1 of~\cite{diakonikolas2025clustering})}
\label{alg:hard_purging}
\begin{algorithmic}[1]
	\REQUIRE Sample set $S=\{(x_i,y_i)\}_{i=1}^n$, parameter $\alpha\in(0,1)$.
	\ENSURE Disjoint subsets $B_j$ of $S$.
	
	\STATE Generate a list $L_\text{stdev}$ of candidate standard deviations.
	
	\STATE
	Apply list-decodable mean estimation 
	to $\{x_i\}$ for each $\hat{\sigma}\in L_\text{stdev}$ and generate a list of candidate means $L_{\mathrm{mean}}$.
	
	\STATE
	Initialise $L \leftarrow \emptyset$.
	
	\STATE
	For every $s\in L_\text{stdev}$ in increasing order:
	
	\begin{enumerate}
		\item[(a)] For every $\mu\in L_{\mathrm{mean}}$:
		\begin{enumerate}
			\item[(i)] If $\|\mu-\hat\mu\|_2>99C\hat{\sigma}/\sqrt\alpha$ for all $\hat\mu\in L$, decide the following convex feasibility program
			\begin{align*}
				&\text{find}\; q_i\in[0,1], i\in[n],\\
				&\text{s.t.}\;
				\Bigl\|\sum_{i=1}^n q_i(x_i-\mu)(x_i-\mu)^\top\Bigr\|_{(1/\alpha)}
				\le\frac{2C^2 \hat{\sigma}^2}{\alpha}\sum_{i=1}^n q_i,\quad
				\sum_{i=1}^n q_i\ge 0.97\alpha n.
			\end{align*}
			\item[(ii)] If feasible, add $\mu$ to $L$.
		\end{enumerate}
	\end{enumerate}
	%
	\STATE
	Apply size-based pruning: iteratively remove any $\hat\mu\in L$
	whose Voronoi cell in $S$ has fewer than $0.96\alpha n$ points,
	and recompute the partition after each removal. 
	
	\STATE
	Apply distance-based pruning: iteratively find any pair
	$(\hat\mu_j,\hat\mu_{j'})$ in $L$ whose filtered Voronoi means
	satisfy
	$\|\mu_{B_j}-\mu_{B_{j'}}\|_2\le 4761C(\sigma_{B_j}+\sigma_{B_{j'}})/\sqrt\alpha$
	and remove the less isolated one.
	
	\STATE
	Let $L''$ be the remaining candidate list.
	Compute the filtered Voronoi partition
	$\{B_1,\ldots,B_j\} \leftarrow \mathrm{FilteredVoronoi}(L'',S)$ using Algorithm~3 of~\cite{diakonikolas2025clustering}. \label{step:filtered-voronoi}
	%
	
	\RETURN $\{B_j\}$.
\end{algorithmic}
\end{algorithm}

Let $S_j$ be the samples whose instance is drawn from mixture component $\D_j$. Algorithm~\ref{alg:hard_purging} ensures that each component $S_j$ is found except for a small portion. That is, the guarantees of Algorithm~\ref{alg:hard_purging} from ~\cite{diakonikolas2025clustering} ensures that there exists a union of elements in $\{B_j\}$ that forms a set of samples $\B_j$ which include most of the elements of $S_j$. That makes Step~\ref{step:alg-pruning} of Algorithm~\ref{alg:svm} favorable.

Following that, in Algorithm~\ref{alg:svm}, we then utilize the label information to identify each set $\B_j$ and remove potential outliers, i.e. the samples whose label disagrees with its neighbors. Step~\ref{step:label-pruning} ensures to fully exploit the information of uncorrupted samples that form large clusters. 
After that, taking union of subsets with agreeing labels (Step~\ref{step:cluster-merge}) helps to identify each $\B_j$. In fact, we show that as long as the noise rate is carefully controlled, this step always returns correct clusters.
After these steps, the boundary between clusters $\B_j$ becomes extremely clear, which benefits the effeciency and effectiveness of the subsequent hinge loss minimizatin program.

	\section{Analysis}
Here we discuss the analysis of Algorithms~\ref{alg:svm}. Our analysis is divided into three parts: (i) pointwise gradient analysis~\ref{sec:pointwise-gradient-analysis-main}, (ii) analysis of Algorithm~\ref{sec:analysis-of-algorithm-main}, and (iii) sample complexity~\ref{sec:sample-complexity-main}. We first show the analysis for pointwise gradient, and then we show the analysis for Algorithm~\ref{alg:svm} and Algorithm~\ref{alg:hard_purging}, followed by the analysis of sample complexity to hold the dense pancake condition.

\subsection{Pointwise Gradient Analysis}\label{sec:pointwise-gradient-analysis-main}

We start from the pointwise analysis for the gradient of hinge loss function.
On a given point $(x,y)$, if the pancake centered at it with width $\tau$ w.r.t. $\hat{w}$ is $\rho$-dense, then any point in this pancake will contribute to a gradient that draws $\hat{w}$ to $w^*$. When this gradient is non-negligible, then $\hat{w}$ cannot be the optimum. As a result, we use this fact to construct a contradiction and show that any optimum $\hat{w}$ should correctly classify any significant point $(x,y)$, i.e. whose pancake $P_{\hat{w}}^{\tau}(x,y)$ is enough dense.

This proof idea is similar to our prior works~\cite{talwar2020error,shen2025efficient}. However, when considering the multiclass setting, the extension is not trivial. Recall that the margin condition for the multiclass case takes care of all $(k-1)$ classes other than the true class to which the instance belongs. Similarly, the dense pancake condition ensures that the instances lie close to each other not only in one direction, but in $k$ different directions at the same time. These new conditions drastically change the analysis.

Specifically, the following lemmas show that if $(x,y)$ is misclassified by $\hat{w}$, then projecting any subgradient $g\in \partial_w\ell(w;S)\mid_{w=\hat w}$ onto some specific directions would give us quantities with large magnitudes, showing that the subgradients are significant. Let $\hat{S} =\hat{S}_C\cup \hat{S}_D$ and define $S_P:=\hat{S}_C\cap\Phat(x,y)$.

\begin{lemma}[Gradient bound on $w^*$ direction]\label{lem:wstar}
	Let $\hat{w}$ be an optimum returned by Eq.\eqref{eq:svm} in Algorithm~\ref{alg:svm}.
	Suppose Assumption~\ref{ass:large_margin} holds with parameter $\gamma$.
	Assume that sample $(x,y)$ is misclassified by $\hat{w}$. Consider pancake $\Phat(x,y)$ with $\tau<\gamma/4$, then
	\begin{equation}\label{eq:L18}
		\inner{g}{w^*}
		\leq -|S_P|+\frac{2C\sigma n\sqrt{\eta}}{\gamma},
	\end{equation}
\end{lemma}
\begin{lemma}[Gradient bound on orthogonal direction]\label{lem:wprime}
	Let $\hat{w}$ be an optimum returned by Eq.\eqref{eq:svm} in Algorithm~\ref{alg:svm}. Define $w':=\frac{w^*-\kappa\hat w}{\sqrt{1-\kappa^2}}$~\label{eq:wprime}, where $\kappa:=\inner{\hat w}{w^*}$, and assume that $\theta(\hat w,w^*)\in(0,\pi/2)$. Suppose Assumption~\ref{ass:large_margin} holds with parameter $\gamma$.
	Assume that sample $(x,y)$ is misclassified by $\hat{w}$. Consider pancake $\Phat(x,y)$ with $\tau<\gamma/4$, if $\frac{|S_P|}{4}>\frac{2C\sigma n\sqrt{\eta}}{\gamma}$, then
	\begin{equation}\label{eq:L19}
		\inner{g}{w'}\leq-\frac{\sqrt{5}}{4}|S_P|.
	\end{equation}
\end{lemma}
It is easy to see that it is sufficient to choose $\frac{|S_P|}{4}>\frac{2C\sigma n\sqrt{\eta}}{\gamma}$ for both Eq.~\eqref{eq:L18} and \eqref{eq:L19} to be strictly less than $0$. On the contrary, we utilize the KKT condition and prove that there always exists some $g$ that make $\inner{g}{w^*}$ or $\inner{g}{w'}$ equal to $0$ depending on the optimality of $\hat{w}$, leading to a contradiction. Hence, $(x,y)$ must be correctly classified by $\hat{w}$. We summarize in the following lemma.


\begin{lemma}
	\label{lem:pointwise}
	Suppose Assumption~\ref{ass:large_margin} holds with parameter $\gamma$. Let $(x,y)\in\X\times\Y$ be a sample point.
	If
	\begin{equation}\label{eq:cond}
		|S_P|>\frac{8C\sigma n\sqrt{\eta}}{\gamma},
	\end{equation}
	then $(x,y)$ is not misclassified by $\hat{w}$.
\end{lemma}

\subsection{Analysis of Algorithms}\label{sec:analysis-of-algorithm-main}
In this section we show the detailed analysis of Algorithm~\ref{alg:svm} and Algorithm~\ref{alg:hard_purging}. 
Let $S_j\subseteq S$ represent clean samples in $S$ selected from the $j$-th component of $\D_X$ and labeled by $w^*$ and let $S_C = \cup_{j\in[k]} S_j$. Similarly, $\hat{S_j} \subseteq \hat S$ denotes the remaining clean samples after Step~\ref{step:combine} of Algorithm~\ref{alg:svm}. It is easy to see that $\hat S_C = \cup_{j\in[k]} \hat S_j$.

With a large enough sample size, the clean samples $S_C$ in $S$ satisfy the dense pancake condition (Definition~\ref{def:multiclass_pancake}). However, under the nasty noise condition, the adversary is allowed to inspect and replace up to an $\eta$ fraction of the samples. The following proposition shows that the dense pancake structure is preserved after the cluster-based pruning step.

\begin{proposition}[Empirical dense pancake]\label{prop:empirical-pancake}
	Consider Algorithm~\ref{alg:svm}. Suppose that all $j\in[k]$, the component $S_j$ is found by Algorithm~\ref{alg:hard_purging}. Let $S$ be a set of samples of size at least $\Omega(k^2\cdot(d\log d +\frac{k}{\epsilon\delta}))$ from $\EXnas$, then with probability at least $1-\delta$, the set $\hat{S}$ is $(\tau,\frac{3\alpha}{20},\epsilon)$-dense with respect to $(\D_X,w^*)$ with appropriate $\tau\in(0,\frac{\gamma}{4})$. In addition, $\lvert{\hat S}\rvert\geq \frac{\abs{S}}{2}$ and the $\frac{3\alpha}{20}$-density is contributed by $\hat S_C$.
\end{proposition}
That is, we utilize results from~\cite{diakonikolas2025clustering} for Algorithm~\ref{alg:hard_purging}, where the clean samples are retained while dirty samples are removed with high probability. That means, if $S_C$ is $(\tau,\rho,\epsilon)$-dense, with high probability, after Step~\ref{step:alg-pruning}, $\hat S_C$ is $(\tau,\rho',\epsilon)$-dense with a slightly smaller $\rho'$. The analysis for this proposition includes some details in the proof for sample complexity in Section~\ref{sec:sample-complexity-main}.

In addition, Proposition~\ref{prop:empirical-pancake} works for both the mixture of logconcaves and the mixture of bounded covariance. Specifically,
we assume that $D_X$ is a mixture of logconcaves with pancake width $\tau = 2\sigma\cdot (\log\frac{k}{\epsilon}+1)$, or a mixture of bounded covariance distributions with pancake width   $\tau = 2\sigma\cdot\sqrt{\frac{k}{\epsilon}}$. 
By Theorem~\ref{thm:sample-comp-logconcave} (or \ref{thm:sample-comp-boundedcov} for bounded covariance),
for each $j\in[k]$, $S_j$ is $(\tau,\frac{1-\epsilon}{2},\epsilon)$-dense with respect to $(\D_j,w^*)$ and we ensure the sample size $\lvert S \rvert$ is sufficiently large to satisfy the conditions for Algorithm~\ref{alg:hard_purging} to perform effectively. As a result, the pruning algorithm ensures that for each $S_j$, a large quantity of clean samples is returned in $\B_j$, leading to a dense pancake condition with $\rho'\geq \frac{3\alpha}{20}$. As a result, we have that $\hat{S}$ is $(\tau,\frac{3\alpha}{20},\epsilon)$-dense with respect to $(\D_X,w^*)$. The detailed proof can be seen in Appendix~\ref{sec:svm-app}. 


It remains to derive the noise condition under which the dense pancake condition is preserved and the label information in each sub-cluster $B_{j'}$ is effective for identifying to which $\B_j$ it belongs to.
The following theorem shows that if $\hat{S}$ satisfies the dense pancake property (Proposition~\ref{prop:empirical-pancake}), then the error rate for the classifier $\hat w$ returned by Program~\eqref{eq:svm} is at most $\epsilon$.
\begin{theorem}[Noise rate condition]\label{thm:det_density_condition}
	Consider Algorithm~\ref{alg:svm} and its returned $\hat{w}$. Suppose Assumption~\ref{ass:large_margin} holds with some $\gamma>0$ and Assumption\ref{ass:mixture} holds with $\sigma$. 
	Assume that $\frac{\gamma}{4}>\max(\tau,C\sigma)$ and $\eta\leq\frac{1}{2^{12}k^2}$. 
	If set $\hat{S}$ is $(\tau,\frac{3\alpha}{20},\epsilon)$-dense with respect to $(\D_X,w^*)$ with appropriate $\tau\in(0,\frac{\gamma}{4})$ for some $\alpha\in[\frac{0.6}{k},\frac{1}{k}]$, the $\frac{3\alpha}{20}$-density is contributed by $\hat S_C$, and $\lvert{\hat S}\rvert\geq \frac{\abs{S}}{2}$, 
	then $\err_{\D_X}(\hat{w})\leq\epsilon$.
\end{theorem}
Using the result of Proposition~\ref{prop:empirical-pancake}, we first compute the density contributed by $\hat S_C \geq \frac{3\alpha}{40}n$. Together with the condition established  by the pointwise gradient analysis (Lemma~\ref{lem:pointwise}), when Assumption~\ref{ass:large_margin} holds with some $\gamma>0$ and Assumption~\ref{ass:mixture} holds with $\sigma$, then $\err_{\D_X}(\hat{w})\leq\epsilon$ is returned by Algorithm~\ref{alg:svm}.

We include the computational complexity for Algorithm~\ref{alg:svm} in the following lemma. It shows that the algorithm runs in polynomial time. For more details, see Appendix~\ref{sec:svm-app}.
\begin{lemma}[Polynomial runtime of Algorithm~\ref{alg:svm}]
	\label{lem:svm_runtime}
	Algorithm~\ref{alg:svm} runs in polynomial time. More precisely, the total runtime is
	$\poly\left(n,d,K,\frac{1}{\alpha}\right).	$
\end{lemma}

For completeness, we restate the main theorem. By Proposition~\ref{prop:empirical-pancake} and Theorem~\ref{thm:det_density_condition}, we show that the classifier $\hat{w}$ returned by Algorithm~\ref{alg:svm} has error at most $\epsilon$, with probability at least $1-\delta$, when it is given a sample of size $\Omega\left(k^2\cdot (d\log d+\log\frac{k}{\epsilon\delta})\right)$ from $\EXnas$.


\subsection{Sample Complexity}\label{sec:sample-complexity-main}
In this section, we give the sample complexity for obtaining dense pancake conditions for both the mixture of logconcave distributions and that of the bounded covariance distributions.
Below is the analysis of sample complexity of a single log-concave distribution and a single bounded covariance distribution. For more detailed analysis, see Appendix~\ref{sec:statistical-app}.

\begin{theorem}[Sample complexity for a single log-concave distribution]\label{thm:sample-comp-logconcave}
	Suppose that Assumption~\ref{ass:large_margin} is satisfied with parameter $\gamma$. If distribution $\D_j$ is logconcave with mean $\mu_j$ and covariance $\Sigma_j\preceq\sigma\I_d$, then a set $\overline{S_j}$ of at least {$\Omega\(\frac{1}{1-k\beta}\cdot\( d\log d +\log\frac{1}{k\beta} + \log\frac{1}{\delta}\)\)$} samples from $(\D_j,w^*)$ satisfies $(\tau,\frac{1-k\beta}{2},k\beta)$-dense pancake condition with respect to $(\D_j,w^*)$, for some $\tau=2\sigma(\log\frac{1}{\beta}+1)\leq\frac{\gamma}{2}$.
\end{theorem}

\begin{theorem}[Sample complexity for a single bounded covariance distribution]\label{thm:sample-comp-boundedcov}
	Suppose that Assumption~\ref{ass:large_margin} is satisfied with parameter $\gamma$. If $\D_j$ is a distribution with mean $\mu_j$ and covariance $\Sigma_j\preceq\sigma\I_d$, then a set $\overline{S_j}$ of at least {$\Omega\(\frac{1}{\rho}\cdot\( d\log d +\log\frac{1}{k\beta} + \log\frac{1}{\delta}\)\)$} samples from $(\D_j,w^*)$ satisfies $(\tau,\frac{1-k\beta}{2},k\beta)$-dense pancake condition with respect to $(\D_j,w^*)$, for some $\tau=2\sigma\cdot\frac{1}{\sqrt{\beta}}\leq\frac{\gamma}{2}$.
\end{theorem}

Observing that $\tau = 2\sigma(\log \frac{1}{\beta} + 1)$ grows  logarithmically in $1/\beta$ for the log-concave case, while under the bounded covariance, $\tau = 2\sigma / \sqrt{\beta}$ depends polynomially on $1/\beta$. Hence, the  width is much larger for mixture of bounded covariance distributions (especially for extremely small $\beta$), which makes the margin condition weaker since it forces to $\tau \le \gamma/2$. However, it is worth noting that the sample complexity for both cases is the same. As a result, our result on the sample complexity for empirical dense pancakes shows that it is independent of the standard deviation $\sigma$ or the tail bound.

Conditioned on the above sample complexity for a single distribution component, we are able to obtain the total sample complexity by utilizing Theorem~\ref{thm:sample-comp-mixture}. We simply sum up  the per cluster density to set $\overline{S}$ since each subset  $\overline{S_j}$, meets the dense pancake condition with respect to $(\D_j, w^*) $, and the fact that each cluster weight $q_j \geq \alpha $ implies  $\lvert\overline{S_j}\rvert \geq \frac{\alpha\lvert S\rvert}{2}$.
As a result, $\overline{S}$ meets the $(\tau, \frac{\alpha\rho}{2}, \beta)$-dense pancake condition with respect to the mixture $(\D_X, w^*)$.

\begin{theorem}[Sample complexity for mixtures]\label{thm:sample-comp-mixture}
	Suppose that Assumption~\ref{ass:mixture} is satisfied. If $\forall j\in[K]$, a sample $\overline{S_j}$ of size $\lvert{\overline{S_j}}\rvert \geq n_j$ samples from $(\D_j,w^*)$ satisfies $(\tau,\rho,\beta)$-dense pancake condition with respect to $(\D_j,w^*)$, then by drawing a set $\overline{S}$ of size $\Omega\(\sum_{j=1}^{K}n_j+\frac{1}{\alpha}\log\frac{K}{\delta}\)$ samples from $(\D_X,w^*)$ satisfies $(\tau,\frac{\alpha\rho}{2},\beta)$-dense pancake condition with respect to $(\D_X,w^*)$.
\end{theorem}
This shows that only with a small overhead in the sample complexity, the dense pancake is stable under the mixture of distributions. That is, when each component distribution has a weight $q_i$ lower bounded by $\alpha$, the density is only reduced by a multiplicative factor of $\alpha/2$.

	\section{Conclusion and Limitations}

In this paper, we revisit the problem of learning multiclass linear classifiers and study the PAC guarantees under the nasty noise condition. We show that under both the mixture of bounded covariances and margin conditions, it is possible to design robust algorithms that are tolerant to a constant rate of nasty noise. Consider that under other well studied noise models, such as malicious noise, adversarial label noise, or random classification noise, the adversary is more restricted. Hence, it is natural that our algorithm and analysis can be adapted to more noisy settings.



In our analysis,
the Euclidean norm of the sum of centered data points in $S_D$ is bounded by a term proportional to $C\sigma n\sqrt{\eta}$. Because of this square-root dependence instead of the linear dependence on $\eta$, the noise tolerance becomes weaker, changing from order $O(1/k)$ to $O(1/k^2)$. Though this still gives a constant noise-rate bound, removing this dependence can significantly strengthen the robustness of the algorithm.
This loss is one of the main limitations of second-order methods, and whether a bound of $1/k$ for malicious and nasty noise is achievable is still an open problem.

	\bibliographystyle{alpha}
	\bibliography{references}
	
	\clearpage
	\appendix
	
	\section{Pointwise Gradient Analysis}\label{sec:pointwise-gradient-analysis}
Recall that the feature map $\mcPsi:\X\times\Y\to\mcFeat$ has the $y$-th block equal
to $x$ and all others zero, i.e. $\inner{w}{\mcPsi(x,y)}=\inner{w_y}{x}$.
The multiclass hinge loss is
\begin{equation}
	\ell(w;\,x_i,y_i)
	=\max_{y'\in\Y}\Bigl(\one(y'\neq y_i)
	-\tfrac{1}{\gamma}\inner{w}{\mcPsi(x_i,y_i)-\mcPsi(x_i,y')}\Bigr).
\end{equation}
At the optimum $\hat{w}$,
the subgradient of $\ell$ w.r.t.\ $w$ is $	\frac{1}{\gamma}\bigl(\mcPsi(x_i,\hat{y}_i)-\mcPsi(x_i,y_i)\bigr)$,
where we define a sudo-label for any $i$
\begin{equation}
\hat{y}_i:=\argmax_{y'_i\in\Y}\Bigl(\one(y'_i\neq y_i)
-\tfrac{1}{\gamma}\inner{\hat{w}}{\mcPsi(x_i,y_i)-\mcPsi(x_i,y'_i)}\Bigr),
\end{equation}
which is used in the optimization process.
Therefore, taking subgradient of the $\ell(w;\hat S)$ with respect to $w$ yields
\begin{equation}
	g
	=
	\partial_w \bigg(\sum_{i\in\hat S}\ell(w;x_i,y_i) \bigg) \Big|_{w=\hat w} 
	=
	\frac{1}{\gamma}\sum_{i\in\hat S}\left(\mcPsi(x_i,\hat y_i)-\mcPsi(x_i,y_i)\right).
\end{equation}

In our proof, we will frequently discuss the inner product
\begin{equation}\label{eq:block_feature_identity}
\inner{g}{w}=\frac{1}{\gamma}\sum_{i\in\hat S}\inner{w}{\mcPsi(x,y)-\mcPsi(x,y')}
	=\frac{1}{\gamma}
	\sum_{i\in\hat S}\inner{w_y-w_{y'}}{x},
\end{equation}
which will be used interchangeably.
Recall that for any set $\hat S$, we can view it as a union of a subset of the clean samples and that of the dirty samples, i.e. $\hat S = \hat{S}_C \cup \hat{S}_D$.
To perform pointwise gradient analysis, for any $(x,y)$, we further partition the sample set $\hat S$ into three subsets $\hat S = S_P \cup S_O \cup S_D$.
That is, given any point $(x,y)$ and its pancake
$\Phat(x,y):=\{(x',y')\in\X\times\Y:\inner{\hat{w}}{\mcPsi(x',y')}-\inner{\hat{w}}{\mcPsi(x,y)}\leq\tau, y'=y\}$,
define
\begin{equation}\label{eq:S-partition}
	S_P:=\hat{S}_C\cap\Phat(x,y),\quad
	S_O:=\hat{S}_C\setminus\Phat(x,y),\quad
	S_D:=\hat{S}_D.
\end{equation}
Based on this partition, we can also split the total gradient into three parts, i.e. $g=g_1+g_2+g_3$, where
\begin{align*}
	g_1 &= \frac{1}{\gamma}\sum_{i\in S_P}\bigl(\mcPsi(x_i,\hat{y}_i)-\mcPsi(x_i,y_i)\bigr) ,\\
	g_2 &= \frac{1}{\gamma}\sum_{i\in S_O}\bigl(\mcPsi(x_i,\hat{y}_i)-\mcPsi(x_i,y_i)\bigr) ,\\
	g_3 &= \frac{1}{\gamma}\sum_{i\in S_D}\bigl(\mcPsi(x_i,\hat{y}_i)-\mcPsi(x_i,y_i)\bigr).
\end{align*}


Similar to our prior works, we will prove the robustness of the minimization program and its returned optimum $\hat{w}$ using contradictions~\cite{talwar2020error,shen2025efficient}. That is, assume an instance-label pair $(x,y)$ is misclassified by $\hat{w}$, if $(x,y)$ has a large amount of clean sample lying close to it, we are able to show that $\hat{w}$ must not be the optimum.
We include some proofs from the prior works if there are significant changes, and omit those which we can reuse. We remark that the structure of the proof might be similar to prior works, but many details for the multiclass learning setting (including the new pruning subroutine) can not be naively extended from the binary setting.
\begin{lemma}[Restatement of gradient bound on $w^*$ direction]\label{lem:wstar-restate}
Let $\hat{w}$ be an optimum returned by Eq.\eqref{eq:svm} in Algorithm~\ref{alg:svm}.
Suppose Assumption~\ref{ass:large_margin} holds with parameter $\gamma$.
Assume that sample $(x,y)$ is misclassified by $\hat{w}$. Consider pancake $\Phat(x,y)$ with $\tau<\gamma/4$ and define subsets as in Eq.~\eqref{eq:S-partition}, then
\begin{equation}\label{eq:L18-restate}
	\inner{g}{w^*}
	\leq -|S_P|+\frac{2C\sigma n\sqrt{\eta}}{\gamma},
\end{equation}
\end{lemma}

\begin{proof}
	We bound $\inner{g_1}{w^*}$, $\inner{g_2}{w^*}$, and $\inner{g_3}{w^*}$
	individually, and then combine the bounds.

	\medskip
	\noindent\textbf{Part~1: $\inner{g_1}{w^*}\leq-|S_P|$.}
	The term $g_1$ sums the subgradients of the pancake clean samples $S_P$:
	\begin{equation}
		\inner{g_1}{w^*}
		=\frac{1}{\gamma}\sum_{i\in S_P}\inner{w^*_{\hat{y}_i}-w^*_{y_i}}{x_i}.
	\end{equation}
For each $i\in S_P$, we have that $y_i=y$. Since $(x,y)$ is misclassified, we have that 
$y\neq {y'} \leftarrow \arg\max_{y'\in\Y} \hat{w}_{y'} \cdot x$. 
In addition, due to the dense pancake condition (Definition~\ref{def:multiclass_pancake}) that for any $\bar y$, $\abs{\inner{\hat{w}}{\Psi(x_i,\bar y)-\Psi(x,\bar y)}}\leq \tau$, we can show that $\hat y_i \neq y_i$ for all $i\in S_P$. In more details, there exists $y'_i$ such that
\begin{align*}
\one(y'_i \neq y_i)
-\tfrac{1}{\gamma}\inner{\hat{w}}{\mcPsi(x_i,y_i)-\mcPsi(x_i,y'_i)} > 0,
\end{align*}
and thus $y_i \neq \hat{y}_i$. It suffies to let $y'_i=y_i$ and show
\begin{align*}
\inner{\hat{w}}{\mcPsi(x_i,y_i)-\mcPsi(x_i,y')}  < \gamma.
\end{align*}
The above inequality holds because
\begin{align}
\inner{\hat{w}}{\Psi(x_i,y_i)-\Psi(x_i,y')} &= \inner{\hat{w}}{\Psi(x_i,y_i) - \Psi(x,y) + \Psi(x,y)} - \inner{\hat{w}}{\Psi(x_i,y') - \Psi(x,y') + \Psi(x,y')} \notag\\
&< \inner{\hat{w}}{\Psi(x_i,y_i) - \Psi(x,y)} - \inner{\hat{w}}{\Psi(x_i,y') - \Psi(x,y')} \leq 2\tau < \gamma, \label{eq:exeed-bound}
\end{align}
where the second transition is due to that $\inner{\hat{w}}{\Psi(x,y)-\Psi(x,y')}<0$.
Hence, we know that $\hat{y_i}$ and $y_i$ are different.
By Assumption~\ref{ass:large_margin},
$\inner{w^*_{y_i}-w^*_{\hat{y}_i}}{x_i}\geq\gamma$ for any $i\in\hat{S}_C$,
so each summand $\inner{w^*_{\hat{y}_i}-w^*_{y_i}}{x_i}\leq-\gamma$.
Summing over $S_P$ gives $\inner{g_1}{w^*}\leq-|S_P|$.

\medskip
\noindent\textbf{Part~2: $\inner{g_2}{w^*}\leq0$.}
The term $g_2$ sums the subgradients of the off-pancake clean samples:
\begin{equation}
	\inner{g_2}{w^*}=\frac{1}{\gamma}\sum_{i\in S_O}\inner{w^*_{\hat{y}_i}-w^*_{y_i}}{x_i}.
\end{equation}
For any $i\in S_O\subseteq\hat{S}_C$,
Assumption~\ref{ass:large_margin} again gives
$\inner{w^*_{\hat{y}_i}-w^*_{y_i}}{x_i}\leq -\gamma$ for $\hat{y_i} \neq y_i$. However, as it is unknown whether $\hat{y_i} = y_i$,
each summand is $\leq0$.
Therefore $\inner{g_2}{w^*}\leq0$ .
	
	\medskip
	\noindent\textbf{Part~3: $\inner{g_3}{w^*}\leq -|S_D| + \sqrt{\eta}\cdot C\sigma n$.}
	The term $g_3$ sums the subgradients of the surviving dirty samples:
	\begin{equation}
		\inner{g_3}{w^*}=\frac{1}{\gamma}\sum_{i\in S_D}\inner{w^*_{\hat{y}_i}-w^*_{y_i}}{x_i}.
	\end{equation}
	For each $i\in S_D$, write $x_i=\mu_j+(x_i-\mu_j)$
	where $\mu_j$ is the true mean of the cluster $B_j$ containing $x_i$. Due to Step~\ref{step:label-pruning}, $(x_i,y_i)$ should agree with $(\mu_j,y_j)$ where $y_j=\maj(B_j)$. Since each $B_j$ is of significant size $0.92\alpha n$, $y_j$ must agree with $h_{w^*}(\mu_j)$.
	This gives
\begin{align}
	\sum_{i\in S_D}\inner{w^*_{\hat{y}_i}-w^*_{y_i}}{x_i}
	&=
	\sum_{i\in S_D}\inner{w^*_{\hat{y}_i}-w^*_{y_i}}{\mu_j}
	+
	\sum_{i\in S_D}\inner{w^*_{\hat{y}_i}-w^*_{y_i}}{x_i-\mu_j}
	\label{eq:dirty_decomp_mu}\\
	&=
	\sum_{i\in S_D}\inner{w^*_{\hat{y}_i}-w^*_{y_j}}{\mu_j}
	+
	\sum_{i\in S_D}\inner{w^*_{\hat{y}_i}-w^*_{y_i}}{x_i-\mu_j}
	\label{eq:dirty_label_replace}\\
	&\leq
	-\gamma|S_D|
	+
	2\twonorm{\sum_{i\in S_D}(x_i-\mu_j)}.
	\label{eq:dirty_wstar_bound}
\end{align}
where the second transition is due to that $\forall i, \tnorm{w^*_{\hat{y}_i}-w^*_{y_i}}\leq 2$ since $\tnorm{w^*_y}\leq 1$ from $\tnorm{w^*}\leq 1$. 
To bound the second term, we note that by applying the Step~\ref{step:filtered-voronoi} (Filtered Voronoi), each return cluster $B_j$ has bounded empirical covariance (Theorem~\ref{thm:filtered-voronoi}). That is,
	\begin{equation*}
		\frac{1}{{|B_j|}}\sum_{x\in B_j} (x-\mu_j)(x-\mu_j)^\top \preceq C^2\sigma^2\I_d.
	\end{equation*}
	Then, we have that for any unit vector $v\in\R^d$, $\sum_{i\in S_D \cap B_j} ((x_i-\mu_j)\cdot v)^2 \leq \sum_{i\in B_j} ((x_i-\mu_j)\cdot v)^2 = v^\top \sum_{i\in B_j}(x_i-\mu_j)(x_i-\mu_j)^\top v \leq C^2\sigma^2|B_j|$.
	As a result, $\sum_{i\in S_D} ((x_i-\mu_j)\cdot v)^2 \leq C^2\sigma^2{|\hat S|}$ (Theorem~\ref{thm:filtered-voronoi}).
	Therefore, we can bound 
	\begin{align*}
		\twonorm{\sum_{i\in S_D} x_i-\mu_j} &\leq {\max_{v\in\R^d,\twonorm{v}\leq1} \sum_{i\in S_D} (x_i-\mu_j)\cdot v} \\
		&\leq \sqrt{|S_D|} \cdot \sqrt{\max_{v\in\R^d,\twonorm{v}\leq1} \sum_{i\in S_D} ((x_i-\mu_j)\cdot v)^2} \\
		&\leq \sqrt{\eta n}\cdot C\sigma \sqrt{{|\hat S|}} \\
		&\leq \sqrt{\eta} \cdot C\sigma n,\label{eq:data_norm_bound}
	\end{align*}
	where the second transition is due to Cauchy-Schwarz.
	As a result, we have that \begin{equation}\label{eq:dirty_bound_wstar}
   \inner{g_3}{w^*} \leq - |S_D| + \frac{2}{\gamma}\cdot \sqrt{\eta}\cdot C\sigma n.
	\end{equation}
	
	
\noindent\textbf{Combining.}
Adding the three parts and dropping the non-positive term
$-|S_D|$:
\begin{equation}
\inner{g}{w^*}
=\inner{g_1}{w^*}+\inner{g_2}{w^*}+\inner{g_3}{w^*}
\leq-|S_P| + \frac{2}{\gamma}\cdot \sqrt{\eta}\cdot C\sigma n \ 
\end{equation}
which yields~\eqref{eq:L18}.
\end{proof}

\begin{lemma}[Restatement of gradient bound on orthogonal $w'$ direction]\label{lem:wprime-restate}
Let $\hat{w}$ be an optimum returned by Eq.\eqref{eq:svm} in Algorithm~\ref{alg:svm}. Define $w':=\frac{w^*-\kappa\hat w}{\sqrt{1-\kappa^2}}$~\label{eq:wprime-restate}, where $\kappa:=\inner{\hat w}{w^*}$, and assume that $\theta(\hat w,w^*)\in(0,\pi/2)$. Suppose Assumption~\ref{ass:large_margin} holds with parameter $\gamma$.
Assume that sample $(x,y)$ is misclassified by $\hat{w}$. Consider pancake $\Phat(x,y)$ with $\tau<\gamma/4$ and define subsets as in Eq.~\eqref{eq:S-partition}, if $\frac{|S_P|}{4}>\frac{2C\sigma n\sqrt{\eta}}{\gamma}$, then
%
\begin{equation}\label{eq:L19-restate}
	\inner{g}{w'}\leq-\frac{\sqrt{5}}{4}|S_P|.
\end{equation}

\end{lemma}

\begin{proof}
We again bound each of $\inner{g_1}{w'}$, $\inner{g_2}{w'}$, $\inner{g_3}{w'}$ in turn. Now we  consider the case used in the orthogonal-direction argument, $\theta(\hat w,w^*)\in(0,\pi/2).$ Since both $\hat w$ and $w^*$ are normalized, we have
$\inner{\hat w}{w^*}=\|\hat w\|_2\|w^*\|_2\cos\theta(\hat w,w^*)=\cos\theta(\hat w,w^*).$ Because $\theta(\hat w,w^*)\in(0,\pi/2)$, it follows that
$0<\cos\theta(\hat w,w^*)<1.$
Therefore, $0<\inner{\hat w}{w^*}<1.$

\noindent\textbf{Part~1: $\inner{g_1}{w'}\leq
	-\dfrac{1}{2}|S_P|$.}
The term $g_1$ sums over the pancake clean samples $S_P$:
$\inner{g_1}{w'}=\frac{1}{\gamma}\sum_{i\in S_P}\inner{w'_{\hat y_i}-w'_{y_i}}{x_i}.$
For each $i\in S_P$, expanding $w'=(w^*-\kappa\hat w)/\sqrt{1-\kappa^2}$ gives
\begin{equation}
	\inner{w'_{y_i}-w'_{\hat y_i}}{x_i}=\frac{\inner{w^*_{y_i}-w^*_{\hat y_i}}{x_i}-\kappa\inner{\hat w_{y_i}-\hat w_{\hat y_i}}{x_i}}{\sqrt{1-\kappa^2}}.
\end{equation}
By Assumption~\ref{ass:large_margin} and the analysis in Lemma~\ref{lem:wstar}, $y_i\neq \hat y_i$ and 
$\inner{w^*_{y_i}-w^*_{\hat y_i}}{x_i}\geq\gamma.$
Since $\hat y_i$ is the active maximizer of the hinge loss $\ell(\hat{w},x_i,y_i)$, according to Eq.~\eqref{eq:exeed-bound}, we have that
\begin{equation*}
\inner{\hat{w}}{\Psi(x_i,y_i)-\Psi(x_i,\hat y_i)}\leq \inner{\hat{w}}{\Psi(x_i,y_i)-\Psi(x_i,y')} \leq 2\tau \leq \frac{\gamma}{2}.
\end{equation*}
where $y'$ is as defined in Eq.~\eqref{eq:exeed-bound}.
Therefore, 
$\inner{w'_{y_i}-w'_{\hat y_i}}{x_i}
\geq\frac{\gamma-\kappa\gamma/2}{\sqrt{1-\kappa^2}}
=\frac{(1-\kappa/2)\gamma}{\sqrt{1-\kappa^2}}.$ 
Equivalently,
$\inner{w'_{\hat y_i}-w'_{y_i}}{x_i}
\leq
-\frac{(1-\kappa/2)\gamma}{\sqrt{1-\kappa^2}}.$
Summing over $S_P$ gives
\begin{equation}
\inner{g_1}{w'}\leq	-\frac{1-\kappa/2}{\sqrt{1-\kappa^2}}|S_P|\leq-\frac{1}{2}|S_P|.
\end{equation}
Since optimizing over $0<\kappa<1$ gives
\begin{equation}
\frac{1-\kappa/2}{\sqrt{1-\kappa^2}} \geq \frac{\sqrt{3}}{2} \geq \frac12.
\end{equation}

	\medskip
	\noindent\textbf{Part~2: $\inner{g_2}{w'}\leq0$.}
	The term $g_2$ sums the subgradients of the off-pancake clean samples:
	\begin{equation}
		\inner{g_2}{w'}
		=
		\frac{1}{\gamma}\sum_{i\in S_O}
		\inner{w'_{\hat y_i}-w'_{y_i}}{x_i}.
	\end{equation}
	For any $i\in S_O\subseteq\hat S_C$, we consider two cases. If $\hat y_i=y_i$, then the active hinge label agrees with the true label, so this sample contributes $0$ to $g_2$. If $\hat y_i\neq y_i$, then since $\hat y_i$ is the hinge-active label and the hinge loss is normalized by $1/\gamma$, we have $1-\frac{1}{\gamma}\inner{\hat w_{y_i}-\hat w_{\hat y_i}}{x_i}\geq0$, and hence $\inner{\hat w_{y_i}-\hat w_{\hat y_i}}{x_i}\leq\gamma$.
	 Using $w'=(w^*-\kappa\hat w)/\sqrt{1-\kappa^2}$, we get
	\begin{equation}
		\inner{w'_{y_i}-w'_{\hat y_i}}{x_i}
		=
		\frac{
			\inner{w^*_{y_i}-w^*_{\hat y_i}}{x_i}
			-
			\kappa\inner{\hat w_{y_i}-\hat w_{\hat y_i}}{x_i}
		}{\sqrt{1-\kappa^2}}
		\geq
		\frac{\gamma-\kappa\gamma}{\sqrt{1-\kappa^2}}
		>0.
	\end{equation}
	Therefore $\inner{w'_{\hat y_i}-w'_{y_i}}{x_i}<0$. Thus, if $\hat y_i=y_i$, the summand is $0$, and if $\hat y_i\neq y_i$, the summand is negative. Hence each summand is $\leq0$, giving $\inner{g_2}{w'}\leq0$.

\noindent\textbf{Part~3: $\inner{g_3}{w'}\leq
	-\frac{1}{2}|S_D|+\frac{C\sigma n\sqrt{\eta}}{\gamma}$.}
The term $g_3$ sums over the surviving dirty samples:
\begin{equation}\label{eq:g3_wprime_def}
	\inner{g_3}{w'}=\frac{1}{\gamma}\sum_{i\in S_D}	\inner{w'_{\hat y_i}-w'_{y_i}}{x_i}.
\end{equation}
Using $w'=\frac{w^*-\kappa\hat w}{\sqrt{1-\kappa^2}}$, we have
\begin{equation}\label{eq:wprime_dirty}
	\inner{g_3}{w'}	=\frac{1}{\gamma\sqrt{1-\kappa^2}}\sum_{i\in S_D}\left[\inner{w^*_{\hat y_i}-w^*_{y_i}}{x_i}-\kappa\inner{\hat w_{\hat y_i}-\hat w_{y_i}}{x_i}\right].
\end{equation}
For each $i\in S_D$, write $x_i=\mu_j+(x_i-\mu_j)$, where $\mu_j$ is the true mean of the cluster $B_j$ containing $x_i$. Due to Step~\ref{step:label-pruning}, $(x_i,y_i)$ agrees with $(\mu_j,y_j)$, where $y_j=\maj(B_j)$. Since each $B_j$ is of significant size, $y_j$ agrees with $h_{w^*}(\mu_j)$. Hence
\begin{align}
	\sum_{i\in S_D}\inner{w^*_{\hat y_i}-w^*_{y_i}}{x_i}
	&=
	\sum_{i\in S_D}\inner{w^*_{\hat y_i}-w^*_{y_i}}{\mu_j}+	\sum_{i\in S_D}\inner{w^*_{\hat y_i}-w^*_{y_i}}{x_i-\mu_j}	\\
	&=
	\sum_{i\in S_D}\inner{w^*_{\hat y_i}-w^*_{y_j}}{\mu_j}+\sum_{i\in S_D}\inner{w^*_{\hat y_i}-w^*_{y_i}}{x_i-\mu_j}	\\
	&\leq
	-\gamma |S_D|+2\twonorm{\sum_{i\in S_D}(x_i-\mu_j)}.\label{g3_wprime_wstar_decomp}
\end{align}
Here the second term is controlled exactly as in the $w^*$-direction calculation. By applying Step~\ref{step:filtered-voronoi} and using the bounded empirical covariance of each returned cluster, we get
\begin{equation}\label{eq:g3_wprime_deviation_bound}
	\twonorm{\sum_{i\in S_D}(x_i-\mu_j)}
	\leq
	C\sigma n\sqrt{\eta}.
\end{equation}
Combining~\eqref{g3_wprime_wstar_decomp} and~\eqref{eq:g3_wprime_deviation_bound} gives the dirty-sample bound in the $w^*$ direction:
\begin{equation}\label{eq:g3_wprime_wstar_bound}
	\frac{1}{\gamma}\sum_{i\in S_D}\inner{w^*_{\hat y_i}-w^*_{y_i}}{x_i}\leq-|S_D|	+\frac{2C\sigma n\sqrt{\eta}}{\gamma}.
\end{equation}
Next, since $\hat y_i$ is the active maximizer of the multiclass hinge loss, its hinge value is at least the correct-label value $0$. Therefore,
$\one(\hat y_i\neq y_i)-\frac{1}{\gamma}\inner{\hat w_{y_i}-\hat w_{\hat y_i}}{x_i}\geq 0.$
Hence
$\inner{\hat w_{\hat y_i}-\hat w_{y_i}}{x_i}\geq-\gamma.$
Since $0<\kappa<1$, this implies
$-\kappa\inner{\hat w_{\hat y_i}-\hat w_{y_i}}{x_i}\leq\kappa\gamma.$

Now substituting~\eqref{eq:g3_wprime_wstar_bound} into~\eqref{eq:wprime_dirty}, we obtain
\begin{align*}
	\inner{g_3}{w'}
	&\leq
	-\frac{1}{\sqrt{1-\kappa^2}}|S_D|+\frac{1}{\sqrt{1-\kappa^2}}\cdot\frac{2C\sigma n\sqrt{\eta}}{\gamma}+\frac{\kappa}{\sqrt{1-\kappa^2}}|S_D|\\
	&\leq
	-\frac{1-\kappa}{\sqrt{1-\kappa^2}}|S_D|+\frac{1}{\sqrt{1-\kappa^2}}\cdot\frac{2C\sigma n\sqrt{\eta}}{\gamma}\\
	&\leq
	-\frac{1}{2}|S_D|+\frac{1}{\sqrt{1-\kappa^2}}\cdot\frac{2C\sigma n\sqrt{\eta}}{\gamma}.
\end{align*}

\medskip
\noindent\textbf{Combining.}
Adding the three parts,
$	\inner{g}{w'}	=	\inner{g_1}{w'}	+	\inner{g_2}{w'}	+	\inner{g_3}{w'}.$
Using
$\inner{g_1}{w'}\leq-\frac12|S_P|$,
$\inner{g_2}{w'}\leq0$, and after dropping the negative  bad samples from $\inner{g_3}{w}$ we get 
$\inner{g_3}{w'}\leq\frac{C\sigma n\sqrt{\eta}}{\gamma}$, hence
\begin{equation}
	\inner{g}{w'}
	\leq
	-\frac{1-\kappa/2}{\sqrt{1-\kappa^2}}|S_P|
	+
	\frac{1}{\sqrt{1-\kappa^2}}
	\cdot
	\frac{2C\sigma n\sqrt{\eta}}{\gamma}.
\end{equation}
Since $\frac{|S_P|}{4}>\frac{2C\sigma n\sqrt{\eta}}{\gamma}$, then
\begin{align}
	\inner{g}{w'}
	&\leq
	-\frac{1-\kappa/2}{\sqrt{1-\kappa^2}}|S_P|
	+
	\frac{|S_P|}{4\sqrt{1-\kappa^2}} \\
	&=
	-\frac{3/4-\kappa/2}{\sqrt{1-\kappa^2}}|S_P|.
\end{align}
Since
$
\min_{\kappa \in (0,1)} \frac{\tfrac{3}{4} - \tfrac{\kappa}{2}}{\sqrt{1-\kappa^2}}
= \frac{\sqrt{5}}{4},
$
we get
\begin{equation}
	\inner{g}{w'}
	\leq
	-\frac{\sqrt{5}}{4}|S_P|.
\end{equation}
\end{proof}

Following main deterministic result gives us the guarantee  that satisfies all the required conditions Assumption~\ref{ass:large_margin},Assumption~\ref{ass:log_concave} and Defination~\ref{def:multiclass_pancake}, such that a point $(x,y) $ will not be misclassified by the classifier $\hat{w}$. It follows the same structure as of  theorem 11  of ~\cite{shen2025efficient}, but we have included  for the completeness.

\begin{lemma}[Restatement of Lemma~\ref{lem:pointwise}]
	\label{lem:pointwise-restatement}
	Suppose Assumption~\ref{ass:large_margin} holds with parameter $\gamma$. Let $(x,y)\in\X\times\Y$ be a sample point. 
	If
	\begin{equation}\label{eq:cond_1}
		|S_P|>\frac{8C\sigma n\sqrt{\eta}}{\gamma},
	\end{equation}
	then $(x,y)$ is not misclassified by $\hat{w}$.
\end{lemma}

\begin{proof}
	Here $\hat w\in\argmin_{\|w\|_2\leq1}
	\sum_{i\in\hat S}\ell(w;x_i,y_i)$ is the optimization problem.
	Assume for contradiction that $(x,y)$ is misclassified by $\hat w$.
	By Lemma~\ref{lem:wstar} and condition on Eq~\eqref{eq:cond},
	$\inner{g}{w^*}<0$ for all
	$g\in\partial_w\ell(w;\hat{S})|_{w=\hat{w}}$. We analyze the conditions in two cases.
	
	\smallskip
	\noindent\textbf{Case~1} ($\tnorm{\hat{w}}<1$).\;
	$\hat{w} $ is in the interior of the constraint set ,
	For the optimization problem, first-order optimality gives
	$0\in\partial_w\ell(w;\hat{S})|_{w=\hat{w}}$, so there exists
	$g$ with $\inner{g}{w^*}=0$, contradicting $\inner{g}{w^*}<0$.
	Thus $(x,y)$ cannot be misclassified.
	
%
	
	\smallskip 
	\noindent\textbf{Case~2} ($\tnorm{\hat{w}}=1$).\;
	$\hat{w}$ is on the boundary of the constraint set .
	By KKT(Karush–Kuhn–Tucker) condition , there exist a subgradient $g \in \partial_w\ell(w;\hat{S})|_{w=\hat{w}}$ and a multiplier $\lambda \ge 0$ such that $g + \lambda \hat{w} = 0.$Taking inner product with $w^*$  gives $\lambda \inner{\hat{w}}{w^*} = -\inner{g}{w^*}$. Since $\inner{g}{w^*} < 0$, we obtain $\lambda \inner{\hat{w}}{w^*} > 0$. Because the $\lambda  \geq 0$ this  gives ,
	\begin{align*}
	\inner{\hat{w}}{w^*} > 0,
	\end{align*}
	so the angle satisfies $\theta(\hat{w}, w^*) \in (0, \pi/2)$.Now define $w'$ as in Eq~\eqref{eq:wprime}.By the orthogonal relation between $\hat{w}$ and $w'$ $\inner{w'}{\hat{w}} = 0.$ Applying the KKT condition again $\inner{g}{w'} = -\lambda \inner{\hat{w}}{w'} = 0.$  However, Lemma~\ref{lem:wprime} with condition on Eq~\eqref{eq:cond}  suggests that ,
	\begin{align*}
		\inner{g}{w'} < 0 \quad \text{for all valid } g \in\partial_w\ell(w;\hat{S})|_{w=\hat{w}},
	\end{align*}
	which contradicts the previous equality. Thus $(x,y)$ cannot be misclassified. Both cases gives  a contradiction, so $(x,y)$ is correctly classified by $\hat{w}$.
\end{proof}

%

	\section{Analysis of Algorithm~\ref{alg:svm}}\label{sec:svm-app}

Let $S_j\subseteq S$ be the clean samples in $S$ drawn from the $j$-th component of $\D_X$ and labeled by $w^*$. That is, $\cup_{j\in[K]}S_j = S_C$. Similarly, denote by $\hat{S_j} \subseteq \hat S$ the remained clean samples after Step~\ref{step:combine} of Algorithm~\ref{alg:svm}. Similarly, $\cup_{j\in[K]}\hat S_j = \hat S_C$.

\begin{proposition}[Restatement of empirical dense pancake, Proposition~\ref{prop:empirical-pancake}]\label{prop:empirical-pancake-restatement}Consider Algorithm~\ref{alg:svm}. 
Let $S$ be a set of samples of size at least $\Omega(k^2\cdot(d\log d +\frac{k}{\epsilon\delta}))$ from $\EXnas$, then with probability at least $1-\delta$, the set $\hat{S}$ is $(\tau,\frac{3\alpha}{20},\epsilon)$-dense with respect to $(\D_X,w^*)$ with appropriate $\tau\in(0,\frac{\gamma}{4})$. In addition, $\lvert{\hat S}\rvert\geq \frac{\abs{S}}{2}$ and the $\frac{3\alpha}{20}$-density is contributed by $\hat S_C$.
\end{proposition}
\begin{proof}
Due to the choice of sample size and Theorem~\ref{thm:sample-comp-mixture} together with Theorem~\ref{thm:sample-comp-logconcave} (for logconcaves) or \ref{thm:sample-comp-boundedcov} (for bounded covariance), the sample set $S_C$ is $(\tau,\Theta(\alpha(1-\epsilon)),\epsilon)$-dense with respect to $(\D_X,w^*)$. To see this, take logconcave as an example, let $\epsilon=k\beta$, Theorem~\ref{thm:sample-comp-logconcave} implies that it suffices to draw $n_j\geq\Omega(\frac{1}{1-\epsilon}\cdot(d\log d + \log\frac{1}{\epsilon\delta}))$ samples to ensure $S_j$ to be $(\tau,\frac{1-\epsilon}{2},\epsilon)$-dense with respect to $(\D_j,w^*)$. Applying Theorem~\ref{thm:sample-comp-mixture}, it suffices to choose 
\begin{align*}
\abs{S_C} &\geq \Omega\(\sum_{j=1}^{k}n_j+\frac{1}{\alpha}\log\frac{k}{\delta}\) \\
&\geq \Omega\(\frac{k}{1-\epsilon}\cdot\(d\log d + \log\frac{1}{\epsilon\delta}\) + \frac{1}{\alpha}\log\frac{k}{\delta} \) \\
&\geq \Omega\(k\cdot\(d\log d +\frac{k}{\epsilon\delta}\)\).
\end{align*}

Note that we adjust parameter $\tau$ based on the distribution assumption: if $\D_X$ is a mixture of logconcaves, $\tau = 2\sigma\cdot (\log\frac{k}{\epsilon}+1)$; if $\D_X$ is a mixture of bounded covariance distributions, then, $\tau = 2\sigma\cdot\sqrt{\frac{k}{\epsilon}}$. Moreover, for each $j\in[k]$, $S_j$ is $(\tau,\frac{1-\epsilon}{2},\epsilon)$-dense with respect to $(\D_j,w^*)$. In addition, to ensure the conditions in Theorem~\ref{thm:pruning} are satisfied, we also require $\abs{S}\geq \Omega(k^2\cdot(d\log d +\frac{k}{\alpha\delta}))$, resulting in the sample size defined in the proposition.

By Lemma~\ref{lem:true-cluster-found}, all $k$ clusters are found as $\B_j$. 
Then, for all $j\in[k]$, $\abs{\hat S_j} = \abs{S_j \cap \B_j} \geq \abs{S_j}-0.045\abs{S_j} = 0.955\abs{S_j}$ (Theorem~\ref{thm:pruning}).
Thus, if for some $(x,y)$, it holds that
\begin{equation*}
	\sum_{(x_i,y_i)\in S_j} \one\left\{ (x_i,y_i)\in\P_{w}^\tau(x,y) \right\} \geq \frac{(1-\epsilon)}{2}\cdot\abs{S_j}
\end{equation*}
then,
\begin{equation*}
	\sum_{(x_i,y_i)\in \hat S_j} \one\left\{ (x_i,y_i)\in\P_{w}^\tau(x,y) \right\} \geq \frac{(1-\epsilon)}{2}\cdot\abs{S_j} - 0.045\abs{S_j}.
\end{equation*}
Let the right hand side be $\rho'\cdot\abs{\hat{S}_j}$. We have that 
\begin{align*}
	\rho' = \frac{\frac{(1-\epsilon)}{2}\cdot\abs{S_j} - 0.045\abs{S_j}}{\abs{\hat{S}_j}} \geq \frac{(1-\epsilon)}{2} - 0.045 \geq \frac{9}{20} - \frac{\epsilon}{2},
\end{align*}
where the second transition is due to that $\abs{\hat{S}_j}\leq \abs{S_j}$. Hence, it is not hard to see that $\hat{S_j}$ is $(\tau,\frac{9}{20}-\frac{\epsilon}{2},\epsilon)$-dense w.r.t. $(\D_j,w^*)$. We then check the union set $\hat S$. Similar to the analysis in Theorem~\ref{thm:sample-comp-mixture} and let $\hat q_i=\frac{\abs{\hat S_j}}{\lvert \hat S \rvert}$ be the empirical weight, we have that 
\begin{equation*}
	\Pr_{(x,y)\sim(\D_X,w^*)} \( \sum_{(x_i,y_i)\in \hat{S}_j}  \one\left\{ (x_i,y_i)\in\P_{w}^\tau(x,y) \right\} \geq \rho'\cdot\lvert{\hat{S}_j}\rvert \)\geq \hat q_i(1-\epsilon),
\end{equation*}
which yields to
\begin{equation*}
	\Pr_{(x,y)\sim(\D_X,w^*)} \( \frac{1}{\lvert{\hat{S}}\rvert} \sum_{(x_i,y_i)\in \hat{S}_j}  \one\left\{ (x_i,y_i)\in\P_{w}^\tau(x,y) \right\} \geq \frac{\rho'\cdot\lvert{\hat{S}_j}\rvert}{\lvert{\hat{S}}\rvert} \)\geq 
	\hat q_i(1-\epsilon).
\end{equation*}
For the density lower bound, we utilize results from Theorem~\ref{thm:pruning}. That is,
\begin{equation*}
	\frac{\lvert{\hat S_j}\rvert}{\lvert{\hat{S}}\rvert} = \frac{\lvert{\hat S_j}\rvert}{\lvert{\B_j}\rvert}\cdot \frac{\lvert{\B_j}\rvert}{\lvert{\hat{S}}\rvert} 
	= \frac{\lvert{S_j \cap \B_j}\rvert}{\lvert{\B_j}\rvert}\cdot \frac{\lvert{\B_j}\rvert}{\lvert{\hat{S}}\rvert} 
	\geq \frac{0.955\lvert{S_j}\rvert}{1.105\lvert{S_j}\rvert} \cdot \frac{0.92\alpha \lvert{{S}}\rvert}{\lvert{{S}}\rvert} 
	\geq \frac34 \alpha,
\end{equation*}
where the third transition is due to that $\abs{S_j\cap \B_j} \geq {0.955}\abs{S_j}$, $\abs{S_j \triangle \B_j} \leq 0.045\abs{S_j}+0.03\alpha\abs{S}\leq 0.045\abs{S_j}+0.06\abs{S_j}\leq0.105\abs{S_j}$ because $\alpha\abs{S}\leq 2\abs{S_j}$, $\B_j\geq0.92\alpha\abs{S}$, and $\abs{\hat S}\leq\abs{S}$. Hence, since it only makes sense to learn some model with error rate $\epsilon<\frac12$, we have that
\begin{equation*}
	\frac34\alpha\cdot \rho' = \frac34\alpha \cdot\(\frac{9}{20}-\frac{\epsilon}{2}\) \geq \frac{3}{20}\alpha.
\end{equation*}
Finally, summing over all clean sample subsets $\hat{S}_j$ yields
\begin{equation*}
	\Pr_{(x,y)\sim(\D_X,w^*)} \( \frac{1}{\lvert{\hat{S}}\rvert} \sum_{(x_i,y_i)\in \hat{S}_C}  \one\left\{ (x_i,y_i)\in\P_{w}^\tau(x,y) \right\} \geq \frac{\rho'\cdot\lvert{\hat{S}_j}\rvert}{\lvert{\hat{S}}\rvert} \)\geq (1-\epsilon).
\end{equation*}
As a result, we have that $\hat{S}$ is $(\tau,\frac{3\alpha}{20},\epsilon)$-dense with respect to $(\D_X,w^*)$. Note that since $\cup_{j\in[k]}\hat S_j = \hat S_C$, the probability density $\frac{3\alpha}{20}$ is all contributed by $\hat S_C$.

In addition, it is easy to see that $\lvert{\hat S}\rvert = \sum_{j} \lvert{\B_j}\rvert \geq k\cdot 0.92\alpha \lvert{S}\rvert \geq 0.92 k\cdot \frac{0.6}{k}\lvert{S}\rvert \geq \lvert{S}\rvert/2$. The proof is complete.
\end{proof}

\begin{theorem}[Noise rate condition]\label{thm:det_density_condition-restate}
Consider Algorithm~\ref{alg:svm} and its returned $\hat{w}$. Suppose Assumption~\ref{ass:large_margin} holds with some $\gamma>0$ and \ref{ass:mixture} holds with $\sigma$. 
Assume that $\frac{\gamma}{4}>\max(\tau,C\sigma)$ and $\eta\leq\frac{1}{2^{12}k^2}$. 
If set $\hat{S}$ is $(\tau,\frac{3\alpha}{20},\epsilon)$-dense with respect to $(\D_X,w^*)$ with appropriate $\tau\in(0,\frac{\gamma}{4})$ for some $\alpha\in[\frac{0.6}{k},\frac{1}{k}]$, the $\frac{3\alpha}{20}$-density is contributed by $\hat S_C$, and $\lvert{\hat S}\rvert\geq \frac{\abs{S}}{2}$, 
then $\err_{\D_X}(\hat{w})\leq\epsilon$.
%
\end{theorem}
\begin{proof}
If $\hat S$ is $(\tau,\frac{3\alpha}{20},\epsilon)$-dense with respect to $(\D_X,w^*)$ with $\tau<\frac{\gamma}{4}$ and the $\frac{3\alpha}{20}$-density is contributed by $\hat S_C$,
we know that there exists at least $1-\epsilon$ probability mass of $(x,y)$ over the distribution $(\D_X,w^*)$ such that
\begin{equation*}
\abs{S_P} = \abs{\Phat(x,y)\cap \hat S_C} \geq \frac{3\alpha}{20} \cdot \abs{\hat S} \geq \frac{3\alpha}{40} \cdot n.
\end{equation*}
From the condition established in Lemma~\ref{lem:pointwise}, it suffices to choose some such that $\frac{3\alpha}{40} \cdot n>\frac{8C\sigma n\sqrt{\eta}}{\gamma}$. By assumption, $\frac{\gamma}{4}>C\sigma$. The sufficient condition is $\frac{3\alpha}{40} > 2\sqrt{\eta}$. Consider that $\alpha\in[\frac{0.6}{k},\frac{1}{k}]$, then it suffices to choose $\eta \leq \frac{1}{2^{12}k^2}$. 
Moreover, this bound also satisfy the condition of $\eta\leq0.01\alpha$ such that Algorithm~\ref{alg:hard_purging} can successfully run.
\end{proof}

\begin{lemma}\label{lem:true-cluster-found}
Consider Step~\ref{step:cluster-merge} of Algorithm~\ref{alg:svm}. For all $j\in[k]$, $\B_j$ satisfying the guarantees in Theorem~\ref{thm:pruning} are found with probability $1-\delta'$.
\end{lemma}
\begin{proof}
Due to Theorem~\ref{thm:pruning}, after applying Algorithm~\ref{alg:hard_purging}, the indices $[m]$ can be partitioned into $k$ subsets each of which represents the true cluster $\B_j$ drawn from $\D_j$. In other words, the returned set of clusters $\{B_{j'}\}$ can be further merged into a correct $\B_j$. In~\cite{diakonikolas2025clustering}, identifying them is hard because it was in the unsupervised learning setting. However, in supervised learning with label information $y$'s, this becomes much easier. As a result, we show that even without the ``no-large-sub-cluster'' condition, we are still able to find each $\B_j$ correctly.

That is, at Step~\ref{step:label-pruning} of Algorithm~\ref{alg:svm}, we remove any dirty samples that disagree with the majority label within the same cluster $B_{j'}$. As we carefully control the noise rate such that the dirty samples can not form a majority, this step always removes dirty samples. Then, we merge the clusters using label information with Step~\ref{step:cluster-merge}. Any clusters with agreeing labels (with $j$) will be merged into $\B_j$. Due to that the dirty samples can not corrupt the majority label of any cluster $B_{j'}$, we know that each $\B_j$ is correctly found.
\end{proof}

\begin{lemma}[Restatement of Lemma~\ref{lem:svm_runtime}~\ref{alg:svm}]
	\label{lem:svm_runtime_restatement}
	Algorithm~\ref{alg:svm} runs in polynomial time. More precisely, the total runtime is
	$\poly\left(n,d,K,\frac{1}{\alpha}\right).	$
\end{lemma}

\begin{proof}
	First, Algorithm~\ref{alg:hard_purging} runs in $\poly(nd/\alpha)$ time. 
	Both $L_{\mathrm{mean}}$ and $L_{\mathrm{stdev}}$ have size polynomial in $n$ and $1/\alpha$, so $L$ is also polynomial size. 
	Each step removes at least one element from $L$ until termination, so there are only polynomially many steps. 
	Each step takes $\poly(nd/\alpha)$ time. See~\cite{diakonikolas2025clustering} for details.
	
	Next, the label-pruning step	$B_j\leftarrow\{(x,y)\in B_j:y=\maj(B_j)\}	$
	is done by scanning the clusters and keeping only the points with the majority
	label, so this step is polynomial in $n$ and $K$.
	Finally, the optimization problem in Eq.~\eqref{eq:svm} is a convex multiclass
	hinge-loss minimization problem over the Euclidean ball $\|w\|_2\leq1$. Since
	the objective is a sum of maximums of affine functions in $w$, it is convex.
	Therefore Eq.~\eqref{eq:svm} can be solved in polynomial time by
	convex optimization methods from book ~\cite{boyd2004convex}.
	
	Combining the runtime of Algorithm~\ref{alg:hard_purging}, the label-pruning
	step, and the convex multiclass SVM step, Algorithm~\ref{alg:svm} runs in
	$	\poly\left(n,d,K,\frac{1}{\alpha}\right)$	time.
\end{proof}

	\section{Analysis of Algorithm~\ref{alg:hard_purging}}\label{sec:pruning-app}

The following theorems present the guarantees for Algorithm~\ref{alg:hard_purging}, which we include for completeness and adapt to our notations as follows. Here $C>0$ is some large enough universal constant.

\begin{theorem}[Theorem 3.1~\cite{diakonikolas2025clustering}]\label{thm:pruning}
Consider a mixture of distribution on $\R^d$, $\D_X = \sum_{i=1}^{k} q_j\D_j$ with positive $q_j\geq\alpha$ for some known $\alpha\in(0,1)$. Let $\mu_j$ and $\Sigma_j$ be the unknown mean and covariance of $\D_j$, and assume that $\Sigma_j\preceq \sigma_j^2\cdot I_d$ for all $j\in[k]$ and $\twonorm{\mu_{j_1}-\mu_{j_2}} > C^2(\sigma_{j_1}+\sigma_{j_2})/\sqrt{\alpha}$. Let $S$ be a corrupted set of samples from $\D_X$ under the strong contamination model with noise rate no more than {$0.01\alpha$}. Let $S_j$ be the samples from the $j$-th mixture component. Then, with probaiblity $1-\delta$, if $\abs{S} \geq \Omega\(\frac{1}{\alpha^2}\cdot(d\log d +\log \frac{1}{\alpha\delta})\)$, then there exists an algorithm that takes $S,\alpha$ as input, runs in $\poly(\frac{d}{\alpha})$, and outputs a collection of $m\leq {\frac{1}{0.92\alpha} }$ disjoint sets $\{B_{j'}\}_{j'\in[m]}$ such that
\begin{enumerate}
\item $\abs{B_{j'}}\geq {0.92 }\alpha \abs{S}$ for all $j'\in[m]$.
\item The indices $[m]$ can be partitioned into $k$ subsets $\{H_1,\dots,H_k\}$  such that if $\B_j$ are defined as $\B_j=\cup_{j'\in H_j} B_{j'}$, then it holds that
\begin{enumerate}
\item $\abs{S_j\backslash \B_j} \leq {0.045 }\abs{S_j}$ for every $j\in[k]$. \label{part:retained-clean}
\item $\abs{\B_j\backslash S_j} \leq {0.03\alpha }\abs{S}$ for every $j\in[k]$.
\item For any $j\in[k]$ and any $j'\in H_j$, we have $\twonorm{\mu_{B_{j'}}-\mu_j} \leq C\sigma_j\sqrt{\abs{S_{j}}/\abs{B_{j'}}}$.
\item For any pair $j'_1\neq j'_2$, we have that $\twonorm{\mu_{B_{j'_1}}-\mu_{B_{j'_2}}} > 366 C (\sigma_{B_{j'_1}}+\sigma_{B_{j'_2}})/\sqrt{\alpha}$.
\end{enumerate}
\end{enumerate}
\end{theorem}

\begin{theorem}[Filtered Voronoi, Fact 2.4 of~\cite{diakonikolas2025clustering}]\label{thm:filtered-voronoi}
	By applying Algorithm~3 of~\cite{diakonikolas2025clustering}, after the Voronoi partition and a filtering process, if all underlying component $S_j$ is $(C/10,\epsilon)$-stable with respect to $\mu_j$ and $\sigma_j$ for some $C>0$ and $\epsilon\leq\frac{1}{25}$, then the returned sets $\{B_j\}$ satisfies that for all $j\in[m]$, $\twonorm{\mu_{B_j}-\mu_j}\leq C\sigma\sqrt{\epsilon}$ and $\frac{1}{\abs{B_j}}\sum_{x\in B_j} (x-\mu_j)(x-\mu_j)^\top \preceq C^2\sigma^2\cdot\I_d$.
\end{theorem}

As a result, after performing the clustering step, i.e. Algorithm~\ref{alg:hard_purging}, we know that every cluster is well separated from other clusters with their empirical means away from each other by $L_2$ distance at least $366C\cdot\frac{2\sigma}{\sqrt{\alpha}}$. Moreover, by the last step of Filtered Voronoi algorithm, it is guaranteed that each returned cluster $B_{j'}$, $\frac{1}{\abs{B_{j'}}}\sum_{x\in B_{j'}} (x-\mu_j)(x-\mu_j)^\top \preceq C^2\sigma^2\cdot\I_d$.

	\section{Sample Complexity}\label{sec:statistical-app}

In this section, we provide a tighter analysis for the sample complexity required for the dense pancake condition for a mixture of distributions. We remark that the existing statistical analysis in~\cite{talwar2020error} only works for an instance space residing in the unit ball, i.e. $\twonorm{x}\leq1$. Even though it is possible to extend it to constant-radius balls in $\R^d$, the analysis still fails in our setting where the $L_2$ norm is strictly larger than a constant.
Since we have used $S_j$ to denote a clean instance set drawn from $\D_j$, we use $\overline{S}_j$ to denote a set of instance-label pairs drawn from $\D_j$ and then labeled by $w^*$.

The following theorem gives a sample complexity for a mixture of distributions, provided that the sample complexity of obtaining dense pancake condition for each mixture component is already available. We remark that this is different from Theorem~19 of~\cite{talwar2020error}, which only guarantees the dense pancake condition from the distributional point of view. 

\begin{theorem}[Restatement of Theorem~\ref{thm:sample-comp-mixture}]\label{thm:sample-comp-mixture-restate}
Suppose that Assumption~\ref{ass:mixture} is satisfied. If $\forall j\in[K]$, a sample $\overline{S_j}$ of size $\lvert{\overline{S_j}}\rvert \geq n_j$ samples from $(\D_j,w^*)$ satisfies $(\tau,\rho,\beta)$-dense pancake condition with respect to $(\D_j,w^*)$, then by drawing a set $\overline{S}$ of size $\Omega\(\sum_{j=1}^{K}n_j+\frac{1}{\alpha}\log\frac{K}{\delta}\)$ samples from $(\D_X,w^*)$ satisfies $(\tau,\frac{\alpha\rho}{2},\beta)$-dense pancake condition with respect to $(\D_X,w^*)$.
\end{theorem}
\begin{proof}
Since for all $j$, $\overline{S_j}$ satisfies $(\tau,\rho,\beta)$-dense pancake condition with respect to $(\D_j,w^*)$, by Definition~\ref{def:multiclass_pancake} we have
\begin{equation*}
	\Pr_{(x,y)\sim(\D_j,w^*)} \(\frac{1}{\lvert{\overline{S_j}}\rvert}\sum_{(x',y')\in \overline{S_j}}\one\(y'=y \wedge \forall \bar y\in\Y, \left| \inner{w_{\bar y}}{x' - x} \right| \leq \tau \)\geq\rho  \)\geq 1-\beta.
\end{equation*}
Now consider the weighted mixture of distribution. Since for every $j\in[K]$, the weight $q_j\geq\alpha$, if $\abs{S}\geq \frac{8}{\alpha}\cdot\log\frac{K}{\delta}$ is satisfied, we have that for all $j$, $\lvert{\overline{S_j}}\rvert \geq \frac{\alpha\cdot\abs{S}}{2}$. That means,
\begin{equation*}
	\Pr_{(x,y)\sim(q_j\D_j,w^*)} \(\frac{1}{\lvert\overline{S}\rvert}\sum_{(x',y')\in \overline{S_j}}\one\(y'=y \wedge \forall \bar y\in\Y, \left| \inner{w_{\bar y}}{x' - x} \right| \leq \tau  \)\geq \frac{\alpha\rho}{2} \)\geq q_j(1-\beta).
\end{equation*}
Summing over the above probability over all $j$ gives
\begin{equation*}
	\Pr_{(x,y)\sim(\D_X,w^*)} \(\frac{1}{\lvert\overline{S}\rvert}\sum_{(x',y')\in \overline{S}}\one\(y'=y \wedge \forall \bar y\in\Y, \left| \inner{w_{\bar y}}{x' - x} \right| \leq \tau  \)\geq \frac{\alpha\rho}{2} \)\geq 1-\beta,
\end{equation*}
which means that $\overline{S}$ is $(\tau,\frac{\alpha\rho}{2},\beta)$-dense pancake condition with respect to $(\D_X,w^*)$.
\end{proof}

We show the sample complexity for the dense pancake condition for a single log-concave distribution (a single mixture component in Assumption~\ref{ass:log_concave}), and the proof for the mixture of bounded covariance distributions naturally follows.

\begin{theorem}[Restatement sample complexity for a single log-concave distribution~\ref{thm:sample-comp-logconcave}]\label{thm:sample-comp-logconcave-restatement}
Suppose that Assumption~\ref{ass:large_margin} is satisfied with parameter $\gamma$. If distribution $\D_j$ is logconcave with mean $\mu_j$ and covariance $\Sigma_j\preceq\sigma\I_d$, then a set $\overline{S_j}$ of at least {$\Omega\(\frac{1}{1-k\beta}\cdot\( d\log d +\log\frac{1}{k\beta} + \log\frac{1}{\delta'}\)\)$} samples from $(\D_j,w^*)$ satisfies $(\tau,\frac{1-k\beta}{2},k\beta)$-dense pancake condition with respect to $(\D_j,w^*)$ with probability at least $1-\delta$, for some $\tau=2\sigma(\log\frac{1}{\beta}+1)\leq\frac{\gamma}{2}$.
\end{theorem}
\begin{proof}
We first show that $(\D_j,w^*)$ satisfies $(2\sigma(\log\frac{1}{\beta}+1),1-k\beta,k\beta)$-dense pancake condition.

It is well known that for any vector $\bar{w}\in\R^d, \twonorm{\bar{w}}\leq1$, any random vector from a logconcave distribution $\D_j$ satisfies that
\begin{equation}\label{eq:logconcave-bound}
\Pr_{x\sim\D_j} \(\abs{\bar{w}\cdot(x-\mu_j)}\leq \sigma\cdot\(\log\frac{1}{\beta}+1\) \) \geq 1-\beta.
\end{equation}
This result immediately implies that $\forall w\in\R^{kd}, \twonorm{w}\leq1$
\begin{equation*}
\Pr_{x\sim\D_j} \(\forall \bar y\in\Y, \abs{{w}_{\bar y}\cdot(x-\mu_j)}\leq \sigma\cdot\(\log\frac{1}{\beta}+1\) \) \geq 1-k\beta.
\end{equation*}
That means, except for a $\beta$ probability mass, 
\begin{equation}\label{eq:logconcave-bound-wstar}
\abs{w^*_{\bar y}\cdot(x-\mu_j)}\leq \tau/2, \forall \bar y\in\Y
\end{equation}
In the following, we show that for any $w^*$ that satisfies Assumption~\ref{ass:large_margin} with $\gamma$, there is a large probablity mass of $\D_j$ that have the same label with its mean $\mu_j$ under the classification of $w^*$. Note that an $x$ may have a different label than $y_j=\arg\max_{y\in\Y}w_y^*\cdot\mu_j$ under $w^*$ only when $\exists y'\in\Y\backslash y_j$,
\begin{align*}
w^*_{y'}\cdot x &\geq w^*_{y_j}\cdot x,\\
w^*_{y'}\cdot \mu_j + w^*_{y'}\cdot (x-\mu_j) &\geq w^*_{y_j}\cdot \mu_j + w^*_{y_j}\cdot (x-\mu_j),\\
(w^*_{y_j}-w^*_{y'})\cdot \mu_j + (w^*_{y_j}-w^*_{y'})(x-\mu_j) &\leq 0
\end{align*}
Since $(w^*_{y_j}-w^*_{y'})\cdot \mu_j \geq \gamma$ for any $y'\in\Y\backslash y_j$, it can only happen when ${(w^*_{y_j}-w^*_{y'})\cdot(x-\mu_j)} \leq -\gamma$. 
However, due to Eq.~\eqref{eq:logconcave-bound-wstar}, we know that for any $y_j,y'\in\Y$
\begin{equation*}
	\abs{ (w^*_{y_j}-w^*_{y'})\cdot(x-\mu_j) } \leq \abs{w^*_{y_j}\cdot(x-\mu_j)} + \abs{w^*_{y'}\cdot(x-\mu_j)} \leq\tau.
\end{equation*}
Hence, $\forall y'\in\Y\backslash y_j$,  $(w^*_{y_j}-w^*_{y'})\cdot(x-\mu_j) \geq -\tau\geq-\gamma$. That is, $w^*_{y'}\cdot x \geq w^*_{y_j}\cdot x$ can not happen and $\argmax_{y\in\Y} w^*_y\cdot x=y_j$. More formally,
\begin{equation*}
\Pr_{(x,y)\sim(\D_j,w^*)} \( y=y_j \wedge \forall \bar y\in\Y, \left| \inner{w_{\bar y}}{x - \mu_j} \right| \leq \tau/2 \)\geq 1-k\beta.
\end{equation*}
Moreover,
\begin{equation}
\Pr_{(x,y)\sim(\D_j,w^*)} \(\Pr_{(x',y')\sim(\D_j,w^*)} \(y'=y \wedge \forall \bar y\in\Y, \left| \inner{w_{\bar y}}{x' - x} \right| \leq \tau \)\geq 1-k\beta \)\geq 1-k\beta.
\end{equation}


We then show the sample complexity for obtaining an empirical set that satisfies a dense pancake condition for $(\D_j,w^*)$. For the ease of presentation, let $\rho=1-k\beta$ and $\beta'=k\beta$, then for any $w\in\R^{kd}, \twonorm{w}\leq1$,
\begin{equation*}
\Pr_{(x,y)\sim(\D_j,w^*)} \( \P_{w}^\tau(x,y)\text{ is $\rho$-dense w.r.t.}(\D_j,w^*) \)\geq 1-\beta',
\end{equation*}
due to Definition~\ref{def:multiclass_pancake}. 

From now on, we use $\D$ instead of $(\D_j,w^*)$ to further ease the notation. The following part follows the idea of~\cite{talwar2020error}, but we include for clarity and completeness. Now, fix a $w$. For any $(x,y)\in S^\text{good}:=\{(x,y): \P_{w}^\tau(x,y)\text{ is }\rho\text{-dense w.r.t.}(\D_j,w^*)\}$, apply Chernoff bound (Lemma~\ref{lem:chernoff}), 
\begin{equation*}
\Pr_{S\sim\D^n}\(\frac{1}{n}\sum_{i\in S} \one\((x_i,y_i)\in\P_{w}^\tau(x,y)\) < \rho/2 \) \leq \exp\(-\frac{\rho n}{8}\).
\end{equation*}
Then, 
\begin{equation*}
\E_{S\sim\D^n}\[ \Pr_{(x,y)\sim\D}\( \one\Big\{ (x,y)\in S^\text{good} \wedge \frac{1}{n}\sum_{i\in S}\one\big\{(x_i,y_i)\in\P_{w}^\tau(x,y)\big\} < \rho/2 \Big\} \) \] \leq \exp\(-\frac{\rho n}{8}\).
\end{equation*}
Further take expectation over $\D$ gives
\begin{equation*}
\E_{S\sim\D^n,(x,y)\sim\D}\[ \one\Big\{ (x,y)\in S^\text{good} \wedge \frac{1}{n}\sum_{i\in S}\one\big\{(x_i,y_i)\in\P_{w}^\tau(x,y)\big\} < \rho/2 \Big\}  \] \leq \exp\(-\frac{\rho n}{8}\),
\end{equation*}
which gives
\begin{equation*}
\E_{S\sim\D^n}\[ \Pr_{(x,y)\sim\D}\( (x,y)\in S^\text{good} \wedge \frac{1}{n}\sum_{i\in S}\one\big\{(x_i,y_i)\in\P_{w}^\tau(x,y)\big\} < \rho/2 \) \] \leq \exp\(-\frac{\rho n}{8}\).
\end{equation*}
By taking Markov's inequality, we have
\begin{equation*}
\Pr_{S\sim\D^n}\[ \Pr_{(x,y)\sim\D}\( (x,y)\in S^\text{good} \wedge \frac{1}{n}\sum_{i\in S}\one\big\{(x_i,y_i)\in\P_{w}^\tau(x,y)\big\} < \rho/2 \) > \beta' \] \leq \frac{1}{\beta'}\cdot\exp\(-\frac{\rho n}{8}\).
\end{equation*}
Together with the samples in  complement set of $S^\text{good}$, which contributes to another $\beta'$ probability mass, we have that
\begin{equation*}
\Pr_{S\sim\D^n}\[ \Pr_{(x,y)\sim\D}\( \frac{1}{n}\sum_{i\in S}\one\big\{(x_i,y_i)\in\P_{w}^\tau(x,y)\big\} < \rho/2 \) > 2\beta' \] \leq \frac{1}{\beta'}\cdot\exp\(-\frac{\rho n}{8}\).
\end{equation*}

It remains to show that the above inequality holds for any $w\in\R^{kd}, \twonorm{w}\leq1$. Note that due to the concentration bound of logconcave distribution, a probability mass of at least $1-\delta'/2$, $\twonorm{x-\mu_j}\leq \sigma\sqrt{d}\cdot(\log\frac{2}{\delta'}-1) =: b/2$. 
Now, let us unfix the vector $w$. We note that it suffices to include the pancake $\P_{w}^\tau(x,y)$ in a larger pancake of $\P_{w'}^{\tau+\tau'}(x,y)$ with a different $w'$. That is, we require
\begin{equation*}
\forall (x'y')\in\P_{w}^\tau(x,y), \abs{\inner{w'}{\Psi(x',y')-\Psi(x,y)}} \leq \tau+\tau'.
\end{equation*}
Due to triangle inequality and since $\forall (x'y')\in\P_{w}^\tau(x,y), \abs{\inner{w}{\Psi(x',y')-\Psi(x,y)}} \leq \tau$, it suffices to bound
\begin{equation*}
\abs{\inner{w'-w}{\Psi(x',y')-\Psi(x,y)}}
\end{equation*}
with $\tau'$. Since $(x',y')$ is in the pancake and $y=y'$, $\twonorm{\Psi(x',y')-\Psi(x,y)}\leq b$ with high probability. It suffices to construct a $\frac{\tau'}{b}$-net over the unit ball. Use standard geometric covering, the net is of size $\exp(d\cdot\log d)$ since $\frac{\tau'}{b} = \Theta(\frac{1}{\sqrt{d}})$. Together with the previous failure rate bound, we have that $\forall w \in\R^{kd}, \twonorm{w}\leq1$
\begin{equation*}
\Pr_{S\sim\D^n}\[ \Pr_{(x,y)\sim\D}\( \frac{1}{n}\sum_{i\in S}\one\big\{(x_i,y_i)\in\P_{w}^\tau(x,y)\big\} < \rho/2 \) > 2\beta' \] \leq
\exp\(d\cdot\log d - \frac{\rho n}{8} + \log\frac{1}{\beta'}\).
\end{equation*}
Let the failure rate to be bounded by another $\delta'/2$ gives that 
\begin{equation}
n \geq \Omega\(\frac{1}{\rho}\cdot\( d\log d +\log\frac{1}{k\beta} + \log\frac{1}{\delta'}\)\).
\end{equation}
\end{proof}

The proof for the sample complexity of dense pancake condition for a single distribution with bounded covriance is similar to the above, but with a worse $\tau=2\sigma\cdot\frac{1}{\sqrt{\beta}}$ and $b=2\sigma\cdot\sqrt{\frac{2d}{\beta}}$. That said, since the ratio $\frac{\tau}{b}$ is still the same as that in the proof for a single logconcave, the sample complexity stays the same. We leave it to interested readeres. To see that $\tau=2\sigma\cdot\frac{1}{\sqrt{\beta}}$, note that 
\begin{equation*}
\Pr_{x\sim\D_j}(\abs{(x-\mu_j)\cdot \bar{w}}\geq\sqrt{t})\leq\frac{\sigma^2}{t}.
\end{equation*}
Let the right hand side be $\beta$ and we obtain that $\sqrt{t}=\sigma^2/\beta$.
\begin{theorem}[Restatement of sample complexity for a single bounded covariance distribution ~\ref{thm:sample-comp-boundedcov} ]\label{thm:sample-comp-boundedcov-restatement}
Suppose that Assumption~\ref{ass:large_margin} is satisfied with parameter $\gamma$. If $\D_j$ is a distribution with mean $\mu_j$ and covariance $\Sigma_j\preceq\sigma\I_d$, then a set $\overline{S_j}$ of at least {$\Omega\(\frac{1}{\rho}\cdot\( d\log d +\log\frac{1}{k\beta} + \log\frac{1}{\delta'}\)\)$} samples from $(\D_j,w^*)$ satisfies $(\tau,\frac{1-k\beta}{2},k\beta)$-dense pancake condition with respect to $(\D_j,w^*)$ with probability at least $1-\delta'$, for some $\tau=2\sigma\cdot\frac{1}{\sqrt{\beta}}\leq\frac{\gamma}{2}$.
\end{theorem}


	\section{Useful Lemmas}

\begin{definition}[Multivector feature map]
	\label{def:multivector_map}
	Fix $\X=\mcX$ and $\Y=\mcY$.
	The multivector (class-sensitive) feature map $\mcPsi:\X\times\Y\to\mcFeat$ is
	\begin{equation}
		\mcPsi(x,y) \;:=\; x\otimes e_y
		\;=\; (0,\ldots,0,\underbrace{x}_{\text{$y$-th block}},0,\ldots,0).
	\end{equation}
	For any $w\in\mcFeat$, write $w=(w_1,\ldots,w_k)$ with $w_y\in\mcX$. Then
	\begin{equation}
		\inner{w}{\mcPsi(x,y)} \;=\; \inner{w_y}{x}.
	\end{equation}
	Moreover, $\tnorm{\mcPsi(x,y)}=\tnorm{x}$ for any $y$ and $\tnorm{\mcPsi(x,y')-\mcPsi(x,y)}=\sqrt{2}\tnorm{x}$
	for $y'\neq y$.
\end{definition}

\begin{lemma}[Chernoff bounds]
	\label{lem:chernoff}
	Let $Z_1,Z_2,\ldots,Z_n$ be independent random variables taking values in $\{0,1\}$ and let
	\begin{equation}
		Z \;=\; \sum_{i=1}^n Z_i.
	\end{equation}
	If $\Pr(Z_i=1)\leq \eta$, then for any $\alpha\in[0,1]$,
	\begin{equation}
		\Pr\!\left(Z \geq (1+\alpha)\eta n\right)\;\leq\;\exp\!\left(-\frac{\alpha^2\eta n}{3}\right).
	\end{equation}
	If $\Pr(Z_i=1)\geq \eta$, then for any $\alpha\in[0,1]$,
	\begin{equation}
		\Pr\!\left(Z \leq (1-\alpha)\eta n\right)\;\leq\;\exp\!\left(-\frac{\alpha^2\eta n}{2}\right).
	\end{equation}
\end{lemma}


\begin{lemma}[Properties of isotropic log-concave distributions]
	\label{lem:logconcave_props}
	Let $D$ be an isotropic log-concave distribution over $\mcX$. Then:
	\begin{enumerate}[label=(\arabic*)]
		\item Orthogonal projections of $D$ onto any subspace are isotropic log-concave.
		\item For any unit vector $u\in\mcX$ and any $\alpha>0$,
		\begin{equation}
			\Pr_{x\sim D}\!\left(|\inner{u}{x}|\geq \alpha\right)\;\leq\;e^{-\alpha+1}.
		\end{equation}
		\item For any $\alpha\geq 0$,
		\begin{equation}
			\Pr_{x\sim D}\!\left(\tnorm{x}\geq \alpha\sqrt{d}\right)\;\leq\;e^{-\alpha+1}.
		\end{equation}
	\end{enumerate}
\end{lemma}

\end{document}